\documentclass [smallcondensed] {svjour3}

\usepackage[utf8]{inputenc} 
\usepackage[authoryear]{natbib}
\usepackage[T1]{fontenc}    
\usepackage{url}            
\usepackage{booktabs}       
\usepackage{amsfonts}       
\usepackage{nicefrac}       
\usepackage{microtype}      
\usepackage{array}
\usepackage[dvips]{graphicx}
\usepackage{epsfig,multirow}

\usepackage{amsmath,amsthm,amssymb,bm}
\usepackage{amsfonts,dsfont,color,bbm}
\usepackage{algorithm,algorithmic}
\usepackage[font=footnotesize,labelfont=bf]{caption}
\usepackage{textcomp}
\usepackage{adjustbox}
\usepackage{subcaption}
\usepackage{epstopdf,epsf,epsfig}

\usepackage{setspace}
\usepackage[draft]{hyperref}
\usepackage{float}

\newcolumntype{P}[1]{>{\centering\arraybackslash}p{#1}}
\newtheorem{defn}{Definition}
\newtheorem{assumption}{Assumption}
\newtheorem{thm}{Theorem}

\newtheorem{fact}{Fact}
\usepackage{graphicx}

\allowdisplaybreaks 

\DeclareMathOperator*{\argmax}{arg\,max}
\DeclareMathOperator*{\expect}{{{\rm I\kern-.3em E}}}

\newcommand{\bs}{\boldsymbol}

 \newcommand{\comment}[1]{}

\ifodd 0
\newcommand{\rev}[1]{{\color{blue}#1}}
\else
\newcommand{\rev}[1]{#1}
\fi

\ifodd 0
\newcommand{\cem}[1]{{\color{red}(C: #1)}}
\else
\newcommand{\cem}[1]{}
\fi

\ifodd 0
\newcommand{\omer}[1]{{\color{magenta}(OMER:#1)}}
\else
\newcommand{\omer}[1]{}
\fi

\ifodd 1
\newcommand{\dt}[1]{{\color{red}(Detail: #1)}}
\else
\newcommand{\dt}[1]{}
\fi

\title{\LARGE \bf
	Combinatorial Multi-armed Bandit with Probabilistically Triggered Arms: A Case with Bounded Regret\footnote{A preliminary version of this work will appear in IEEE GlobalSIP 2017.}
}
\titlerunning{Combinatorial Multi-armed Bandit: A Case with Bounded Regret}

\author{A. Ömer Sarıtaç \and Cem Tekin}

\institute{
A. Ömer Sarıtaç \at Department of Industrial Engineering, Bilkent University, Ankara, Turkey, \email{omersaritac@gmail.com} \and
Cem Tekin \at Department of Electrical and Electronics Enginnering, Bilkent University, Ankara, Turkey, \email{cemtekin@ee.bilkent.edu.tr} 
}

\begin{document}

    \maketitle
 
	\begin{abstract}
In this paper, we study the combinatorial multi-armed bandit problem (CMAB) with probabilistically triggered arms (PTAs). Under the assumption that the arm triggering probabilities (ATPs) are positive for all arms, we prove
that a class of upper confidence bound (UCB) policies, named Combinatorial UCB with exploration rate $\kappa$ (CUCB-$\kappa$), and Combinatorial Thompson Sampling (CTS), which estimates the expected states of the arms via Thompson sampling, achieve bounded regret.
In addition, we prove that  CUCB-$0$ and CTS incur $O(\sqrt{T})$ gap-independent regret. These results improve the results in previous works, which show $O(\log T)$ gap-dependent and $O(\sqrt{T\log T})$ gap-independent regrets, respectively, under no assumptions on the ATPs.
Then, we numerically evaluate the performance of CUCB-$\kappa$ and CTS in a real-world movie recommendation problem, where the actions correspond to recommending a set of movies, the arms correspond to the edges between the movies and the users, and the goal is to maximize the total number of users that are attracted by at least one movie.
Our numerical results complement our theoretical findings on bounded regret. 
Apart from this problem, our results also directly apply to the online influence maximization (OIM) problem studied in numerous prior works.
\end{abstract}
\keywords{Combinatorial Multi-armed Bandit, Probabilistically Triggered Arms, Bounded Regret, Movie Recommendation, Online Influence Maximization.}

\section{Introduction}\label{sec:intro}
Multi-armed bandit (MAB) problem is a canonical example of problems that involve sequential decision making under uncertainty that has been extensively studied in the past \citep{thompson1933likelihood,Robbins1952,  Lai1985,Auer2002,Bubeck2012}. This problem proceeds over a sequence of epochs, where the learner selects an arm in each epoch, and receives a reward that depends on the selected arm. 
The learner aims to maximize its cumulative reward in the long run, by estimating the arm rewards using the previous reward observations.
Due to the fact that only the reward of the selected arm is revealed to the learner, in order to maximize its cumulative reward the learner needs to trade-off exploration and exploitation.
In short, exploring a new arm may result in short term loss (due to not selecting the estimated best arm) but long term gain (due to discovering superior arms), while exploiting the estimated best arm may result in short term gain but long term loss (due to failing to detect superior arms). 
%
%

Although theoretically appealing, the classical MAB problem described above is not appropriate for real-world applications where multiple arms are chosen in each epoch, and the resulting reward is a non-linear function of the chosen arms. 
These applications include wireless networking, online advertising and recommendation, and viral marketing, which are studied under the combinatorial multi-armed bandit (CMAB) formalism \citep{anantharam1987asymptotically,Gai2012,chen2013combinatorial,Gopalan2014,Kveton2014,Kveton2015a,Kveton2015c}.
In CMAB, the set of arms chosen by the learner at each epoch is referred to as the action. At the end of each epoch, the learner observes both the reward of the chosen action and the states of the chosen arms. 
This problem is significantly more difficult from the classical MAB problem due to the fact that the size of the action set is combinatorial in the number of arms. 

An interesting extension to the CMAB is CMAB with PTAs \citep{Chen2016a}, where the actions chosen by the learner may trigger arms probabilistically. In this work the authors propose the combinatorial UCB (CUCB) algorithm and prove a $O(\log T)$ gap-dependent regret bound for CUCB. 
%
Later, this model is extended in \citet{Wang2017}, where the authors provide tighter regret bounds by getting rid of a problem parameter $p^*$, which denotes the minimum positive probability that an arm gets triggered by an action. This is achieved by introducing a new smoothness condition on the expected reward function. Based on this, the authors prove $O(\log T)$ gap-dependent and $\tilde{O}(\sqrt{T})$ gap-independent regret bounds.

In this paper, we consider an instance of CMAB with PTAs, where all of the ATPs are positive. For this problem we propose two different learning algorithms: Combinatorial UCB with exploration rate $\kappa$ (CUCB-$\kappa$) and Combinatorial Thompson Sampling (CTS). The first one uses a UCB-based index to form optimistic estimates of the expected states of the arms, while the latter one samples the expected states of the arms from a posterior distribution formed using the past state observations.
Then, we prove that both CUCB-$\kappa$ and CTS achieve $O(1)$ gap-dependent regret for any $\kappa \geq 0$, and both CUCB-$0$ and CTS achieve $O(\sqrt{T})$ gap-independent regret. Here, CUCB-$0$ corresponds to the greedy algorithm which always exploits the best action calculated based on the sample mean estimates of the arm states. 
Although not very common, bounded regret appears in various MAB problems, including some instances of parameterized MAB problems \citep{mersereau2009structured,atan2015global}. However, these works do not conflict with the asymptotic $O(\log T)$ lower bound for the classical MAB problem \citep{Lai1985}, because in these works the reward from an arm provides information on the rewards from the other arms. 
We argue that bounded regret is also intuitive for our problem, because when the arms are probabilistically triggered, it is possible to observe the rewards of arms that never get selected. 	

The contributions of this paper are summarized as follows:
\begin{itemize}
\item We propose a variant of CMAB with PTAs, where all ATPs are strictly positive.
\item We propose a UCB-based algorithm (CUCB-$\kappa$) and a Thompson sampling based algorithm (CTS) for CMAB with PTAs. 
\item We prove that the gap-dependent regrets for CUCB-$\kappa$ and CTS are bounded for CMAB with PTAs, when ATPs are positive. This improves the previous $O(\log T)$ regret bound in the prior work \citep{Chen2016a} that holds under a more general setting.
\item We also prove that the gap-independent regret for CUCB-$0$  and CTS is $O(\sqrt{T})$ for CMAB with PTAs, when ATPs are positive. This also improves the previous $O(\sqrt{T\log T})$ regret bound in the prior work \citep{Chen2016a} that holds under a more general setting.
\item We evaluate the performance of CUCB-$\kappa$ and CTS in a movie recommendation problem with PTAs defined over a bipartite graph, and illustrate how the regret is affected by the size of the action, $p^*$, and $\kappa$.
\end{itemize}

The rest of the paper is organized as follows. Related work is given in Section \ref{sec:related}. Problem description is given in Section \ref{sec:problem}. Two learning algorithms are proposed in Section \ref{sec:algorithms}, and their regrets are analyzed in Section \ref{sec:regret}. The movie recommendation application of the CMAB with PTAs is studied in Section \ref{sec:illustrative}. Concluding remarks are provided in Section \ref{sec:conclusion}. All proofs are given in the appendix.
\section{Related Work}\label{sec:related}
Two key techniques are used for learning in MAB problems: UCB-based index policies \citep{Lai1985} and Thompson (posterior) sampling \citep{thompson1933likelihood}. \citet{Lai1985} introduces UCB-based index policies for the MAB problem and proves a tight logarithmic bound on the asymptotic regret, which establishes asymptotic optimality of the proposed set of policies. Later on, \citet{agrawal1995sample} revealed that sample-mean based index policies can also achieve $O(\log T)$ regret, and \citet{Auer2002} showed that $O(\log T)$ regret is achievable not only asymptotically but also uniformly over time by a very simple sample-mean based index policy. Briefly, a UCB-based index policy constructs an optimistic estimate of the expected reward of each arm by using only the reward observations gathered from that arm but not the other arms, and then, selects the arm with the highest index. Therefore, the arm selection of a UCB-based index policy is deterministic given the entire history. 

Although regret bounds for UCB-based index policies exist for more than several decades, no significant progress is done for Thompson sampling until \citet{Agrawal2012} and \citet{Kaufmann2012}, which show the regret-optimality of Thompson sampling for the classical MAB problem. These efforts to prove regret bounds for Thompson sampling are motivated by works such as \citep{Scott2010, Granmo2010, Graepel2010}, which demonstrate the empirical efficiency of Thompson sampling. Unlike UCB-based index policies, Thompson sampling selects an arm by drawing samples from the posterior distribution of arm rewards. Thus, the arm selection of Thompson sampling is random given the entire history.
In summary, the performance of UCB-based policies and Thompson sampling is well studied in the classical MAB problem where the learner selects one arm at a time.   

On the other hand, in the CMAB problem, most of the prior works consider UCB-based policies. Among these, there exist several works where the reward function is assumed to be a linear function of the outcome of the individual arms \citep{Auer2003,Dani2008,Abbasi2011,Kveton2014,Kveton2015b}. In addition, \citet{Kveton2015a} and \citet{Kveton2015c} solve specific instances of CMAB problems with nonlinear reward functions. Moreover, similar to our work, \citet{Chen2016a} and \citet{Wang2017} consider CMAB with PTAs with a general reward function where the reward function satisfies certain bounded-smoothness and monotonicity assumptions. 

There are also several works that use Thompson sampling based approaches for the CMAB problem and its variants. 
For instance, \citet{Gopalan2014} considers a bandit problem with complex actions, where the reward of each action is a function of the rewards of the individual arms, and derives a regret bound for Thompson sampling when applied to this problem. An empirical study of Thompson sampling for the CMAB problem is carried out in \citet{durand2014thompson}, where it is also used for online feature selection. In addition, recently, Thompson sampling is used in Multinomial Logit (MNL) bandit problem, which involves a combinatorial objective \citep{agrawal2017thompson}. To the best of our knowledge, none of these works are directly applicable to PTAs.
A comparison of our regret bounds with the regret bounds derived in prior works can be found in Table \ref{table:comparison}.

A closely related problem to the CMAB problem is the influence maximization (IM) problem, which is first formulated in \citet{Kempe2003} as a combinatorial optimization problem, and has been extensively studied since then. The goal in this problem is to select a seed set of nodes that maximizes the influence spread in a given network. In this problem, the seed set corresponds to the action, edges of the network correspond to arms, the influence spread corresponds to the reward, and the ATPs are determined by an influence spread process.
Various works consider the online version of this problem, named the OIM problem \citep{Lei2015,Vaswani2015,Wen2017,biz}. In this version, the ATPs are unknown a priori. Works such as \citet{Chen2016a} and \citet{Wang2017} solve this problem by using algorithms developed for CMAB with PTAs. Differently from these, \citet{Lei2015} considers an objective function, which is given as the expected size of the union of nodes influenced in each epoch over time, \citet{Wen2017} adopts the well known algorithm LinUCB for the IM problem and calls it IMLinUCB, which permits a linear generalization, making the algorithm suitable for large-scale problems. \citet{Vaswani2015} introduces the node level feedback in addition to the edge level feedback used in prior works, and proposes a method that uses the node level feedback to update the estimated ATPs. 
In a related work, \citet{biz} introduce the contextual OIM problem, which is a combination of contextual bandits and OIM, and propose an algorithm that achieves sublinear regret.
Importantly, the theoretical results we prove in this paper also applies to the OIM problem defined over a strongly connected graph where each node is reachable from the other nodes in the graph.
{\renewcommand{\arraystretch}{2} 
\begin{table}[h]
{\fontsize{9}{9}\selectfont
			\caption{Comparison of the regret bounds in our work and the relevant literature.}  \label{table:comparison}
		\begin{center}
		\scalebox{0.95}{
			\begin{tabular}{ P{2cm} | P{1cm} | P{2.5cm} | P{2.5cm} | P{2.5cm}  }
				
				& Our Work  & \citet{Chen2016a,Wang2017} & \citet{Kveton2014,Kveton2015c,Kveton2015b,Chen2016b} & \citet{Gopalan2014,Kveton2015a}  \\ \hline
				Gap-dependent Regret & $O(1)$ & $O(\log T)$ & $O(\log T)$ & $O(\log T)$ \\ \hline
				Gap-independent Regret & $O(\sqrt{T})$ & $O(\sqrt{T\log T})$ & $O(\sqrt{T\log T})$ & No Bound \\ \hline
				Strictly positive ATPs & Yes & No & No PTAs & No PTAs   \\ \hline

			\end{tabular}}
			
		\end{center}

}
\end{table}
}
\section{Problem Formulation}\label{sec:problem}
We adopt the notation in \citet{Wang2017}. The system operates in discrete epochs indexed by $t$. There are $m$ arms, given by the set $\{1,\ldots,m\}$, whose states at each epoch are drawn from an unknown joint distribution $D$ with support in $[0,1]^m$. The state of arm $i$ at epoch $t$ is denoted by $X^{(t)}_i$, and the state vector at epoch $t$ is denoted by $\bs{X}^{(t)} := (X_1^{(t)},\ldots, X_m^{(t)})$.

In each epoch $t$, the learner selects an action $S_t$ from the finite set of actions ${\cal S}$ based on its history of actions and observations. 
Then, a random subset of arms $\tau_t \subseteq \{1,\ldots,m\}$ is triggered based on $S_t$ and $\bs{X}^{(t)}$.
Here, $\tau_t$ is drawn from a multivariate distribution (also called the {\em probabilistic triggering function}) $D^{\text{trig}}(S_t,\bs{X}^{(t)})$ with support $[p^*,1]^m$ for some $p^* > 0$, which is equivalent to saying that all the ATPs are positive. As we will show in the subsequent sections, this key assumption allows the learner to achieve bounded regret, without the need for explicit exploration.
Then, at the end of epoch $t$, the learner obtains a finite, non-negative reward $R(S_t,\bs{X}^{(t)},\tau_t)$ that depends deterministically on $S_t$, $\bs{X}^{(t)}$ and $\tau_t$,
and observes the states of the triggered arms, i.e., $X^{(t)}_i$, $i \in \tau_t$. 
The goal of the learner is to maximize its total expected reward over all epochs. 

For each arm $i \in \{1,\ldots,m\}$, we let $\mu_i := \mathbb{E}_{\bs{X}^{(t)} \sim D}[X^{(t)}_i]$ denote the expected state of arm $i$ and $\bs{\mu} := (\mu_1, \ldots, \mu_m)$ denote the {\em expectation vector}. The expected reward of action $S$ is $r_{\bs{\mu}}(S):=\mathbb{E}[R(S,\bs{X},\tau)]$ where the expectation is taken over $\bs{X} \sim D$ and $\tau \sim D^{\text{trig}}(S, \bs{X})$.
We call $r_{\bs{\mu}}(\cdot)$ the {\em expected reward function}.
Let $S^{*}$ denote an optimal action such that $S^{*} \in \argmax_{S \in {\cal S}} r_{\bs{\mu}}(S)$. The expected reward of the optimal action is given by $r^{*}_{\bs{\mu}}$.

Computing the optimal action even when the expected states and the probabilistic triggering function are known is often an NP-hard problem for which $(\alpha,\beta)$-approximation algorithms exist \citep{Vazirani2001}. Due to this, we compare the performance of the learner with respect to an $(\alpha,\beta)$-approximation algorithm ${\cal O}$, which takes $\bs{\mu}$ as input and outputs an action $S^{\cal O}$ such that $\Pr(r_\mu(S^{\cal O})\geq \alpha r^{*}_{\bs{\mu}} ) \geq \beta$. Here, $\alpha$ denotes the approximation ratio and $\beta$ denotes the minimum success probability.
Based on this, the $(\alpha,\beta)$-approximation regret (simply referred to as the regret) of the learner that uses a learning algorithm $\pi$ to select actions by epoch $T$ is defined as follows:
\begin{align}
\text{Reg}_{\bs{\mu},\alpha,\beta}^\pi(T) := T\alpha\beta r^{*}_{\bs{\mu}} - \mathbb{E} \left[\sum_{i=1}^{T} r_{\bs{\mu}}(S_t) \right] \label{eqn:regretdefn}.
\end{align}

For the purpose of regret analysis, as in \cite{Chen2016a}, we impose two mild assumptions on the expected reward function. The first assumption states that the expected reward function is smooth and bounded.
\begin{assumption}[\cite{Chen2016a}]
	$\exists f : \mathbb{R^{+}} \cup \{0\}\rightarrow \mathbb{R^{+}} \cup \{0\}$ such that $f$ is continuous, strictly increasing, and $f(0)=0$, where $f$ is called the bounded smoothness function. For any two expectation vectors, $\bs{\mu}$ and $\bs{\mu}'$, and for any $\Delta > 0$, we have $|r_{\bs{\mu}}(S)-r_{\bs{\mu}'}(S)|\leq f(\Delta)$, if $\max_{i \in \{1,\ldots,m\}} |\mu_i-\mu'_i|\leq \Delta$, $\forall S \in {\cal S}$.
\end{assumption} 
The second assumption states that the expected reward is monotone under $\bs{\mu}$. 
\begin{assumption}[\cite{Chen2016a}]
	If for all arms $i \in \{1,\ldots,m\}$, $\mu_i\leq \mu'_i$, then we have $r_{\bs{\mu}}(S)\leq r_{\bs{\mu}'}(S)$, $\forall S \in {\cal S}$.
\end{assumption} 

\section{Learning Algorithms}\label{sec:algorithms}

In this section we propose a UCB-based learning algorithm called Combinatorial UCB with exploration rate $\kappa$ (CUCB-$\kappa$) and a Thompson sampling based learning algorithm called Combinatorial Thompson Sampling (CTS) for the CMAB problem with PTAs. 

\subsection{CUCB-$\kappa$}\label{sec:CUCB}

\begin{algorithm} [h!]
	\caption{Combinatorial UCB-$\kappa$ (CUCB-$\kappa$)}\label{alg:CUCBkappa}
	\begin{algorithmic}[1]
		
		\STATE \textbf{Input:} Set of actions ${\cal S}$, $\kappa > 0$
		
		\STATE \textbf{Initialize counters:} For each arm $i\in \{1,\ldots,m\}$, set  $T_i=0$, which is the number of times arm $i$ is observed. $t=1$
		
		\STATE \textbf{Initialize estimates:} Set $\overline{\mu}_{i}^\kappa=1$ and $\hat{\mu}_{i}=1$, $\forall i \in \{1,\ldots,m\}$, which are the UCB and sample mean estimates for $\mu_i$, respectively
		\WHILE{$t \geq 1$}

		\STATE Call the $(\alpha,\beta)$-approximation algorithm with $\overline{\bs{\mu}}^\kappa$ as input to get $S_t$ 
		\STATE Select action $S_t$, observe $X_i^{(t)}$'s for $ i \in \tau_t$ and collect the reward $R$  
		
		\FOR{$i \in \tau_t$} 
		
		\STATE $T_i=T_i+1$  
		
		\STATE $\hat{\mu}_i=\hat{\mu}_i+\frac{X_i^{(t)}-\hat{\mu}_i}{T_i}$ 
		\ENDFOR
		\FOR{$i \in \{1,\ldots,m\}$}
		\STATE $\overline{\mu}_i^\kappa= \min\left\{ \hat{\mu}_i+\kappa\sqrt{\frac{3\ln t}{2T_i}},1 \right\}$ 
		\ENDFOR
		\STATE $t=t+1$
		
		\ENDWHILE
	\end{algorithmic}
\end{algorithm}

The pseudocode of CUCB-$\kappa$ is given in Algorithm \ref{alg:CUCBkappa}. CUCB-$\kappa$ is almost the same as CUCB algorithm given in \cite{Chen2016a}, with the exception that the inflation term (adjustment term) is multiplied by a scaling factor $\kappa \geq 0$.
CUCB-$\kappa$ keeps a counter $T_{i}$, which tracks the number of times each arm $i$ is played as well as the sample mean and the UCB estimate of its expected states, denoted by $\hat{\mu}_{i}$ and $\overline{\mu}_i^\kappa$, respectively. Let $\hat{\bs{\mu}} := \{\hat{\mu}_{1},\ldots,\hat{\mu}_{m}\}$ and $\overline{\bs{\mu}}^{\kappa} := \{\overline{\mu}_{1}^\kappa,\ldots,\overline{\mu}_{m}^{\kappa}\}$ denote the estimated expectation vector and the UCB for the expectation vector, respectively. We will use superscript $t$ when explicitly referring to the counters and estimates that CUCB-$\kappa$ uses at epoch $t$. For instance, we have $T_{i}^t=\sum_{j=1}^{t-1} 1_{\{i \in \tau_j\}}$, $\hat{\mu}_i^t=\frac{1}{T_i^t}\sum_{j=1}^{t-1} X_{i}^{(j)}1_{\{i \in \tau_j\}}$ and 
$\overline{\mu}_i^{\kappa,t}=\hat{\mu}_i^t+\min{\{\kappa\sqrt{\frac{3\log t}{2T_i^t}},1\}}$. Initially, CUCB-$\kappa$ sets $T^1_{i}$ as 0 and $\overline{\mu}_i^{\kappa,1}$ as 1 for all arms. Then, in each epoch $t \geq 1$, it calls an $(\alpha,\beta)$-approximation algorithm, which takes as input  $\overline{\bs{\mu}}^{\kappa,t}$ and chooses an action $S_t$. The action $S_t$ depends on the randomness of the approximation algorithm itself in addition to $\overline{\bs{\mu}}^{\kappa,t}$. After playing the action $S_t$, the states of the arms in $i \in \tau_t$, i.e., $\{ X^{(t)}_i \}_{i \in \tau_t}$, are revealed, and a reward $R$ that depends on $S_t$, $\bs{X}^{(t)}$ and $\tau_t$ is collected by the learner. Then, CUCB-$\kappa$ updates its estimates $\hat{\bs{\mu}}^{t+1}$ and $\overline{\bs{\mu}}^{\kappa,t+1}$ for the next epoch based on $\tau_t$ and $\{ X^{(t)}_i \}_{i \in \tau_t}$.

When $\kappa=0$, the inflation term in CUCB-$\kappa$ vanishes. In this case, the algorithm always selects an action that is produced by an $(\alpha,\beta)$-approximation algorithm that takes as input the estimated expectation vector. Hence, CUCB-$\kappa$ becomes a greedy algorithm that always exploits based on the current values of the estimated parameters. Although such an algorithm will incur high regret in the classical MAB problem, we will show that CUCB-$0$ performs surprisingly well in our problem due to the fact that the ATPs are all positive.

\subsection{CTS}\label{sec:CTS}

\begin{algorithm}		
	\caption{CTS} \label{alg:CTS}
	\begin{algorithmic}[1]
		
		\STATE \textbf{Input:} Set of actions ${\cal S}$
		
		\STATE \textbf{Initialize counters:} For each arm $i\in \{1,\ldots,m\}$, set $s_i=0$ and $f_i=0$, which are the success and failure counts of arm $i$. $t=1$

		\WHILE{$t \geq 1$}
		\STATE For each arm $i\in \{1,2,\ldots,m\}$, sample $\nu_i$ from the Beta$(s_i+1,f_i+1)$ distribution.
		\STATE Call the $(\alpha,\beta)$-approximation algorithm with  $\bs{\nu}$ as input to get $S_t$.
		\STATE Select action $S_t$, observe $X_i^{(t)}$'s for $ i \in \tau_t$ and collect the reward $R$. 
		
		\FOR{$i \in \tau_t$} 
		\IF {$X_i^{(t)}=1$} 
		\STATE $s_i=s_i+1$
		\ELSE
		\STATE $f_i=f_i+1$ 
		\ENDIF
		\ENDFOR
		\STATE $t=t+1$
		
		\ENDWHILE
	\end{algorithmic}
\end{algorithm}

The pseudocode of CTS is given in Algorithm \ref{alg:CTS}. For the simplicity of exposition, for CTS, in addition to the assumptions given in Section \ref{sec:problem}, we also assume that $X^{(t)}_i \in \{0,1\}$ for all $i \in \{1,\ldots,m\}$, i.e., the states of the arms are Bernoulli random variables. Note that CTS can easily be generalized to the case when $X^{(t)}_i \in [0,1]$ for all $i \in \{1,\ldots,m\}$, by performing a Bernoulli trial for any arm $i \in \tau_t$ with success probability $X^{(t)}_i$, in way similar to the extension described in \citet{Agrawal2012}.
Also note that Bernoulli arm states are very common, and appear in the OIM problem and the movie recommendation example that we discuss in this paper.

For each arm $i$, CTS keeps two counters $s_{i}$ and $f_{i}$, which count the number of times the state of arm $i$ is observed as $1$ (success) and $0$ (failure), respectively. We denote by $\bs{\nu} : =\{\nu_{1},\ldots,\nu_{m}\}$, the estimated expectation vector where $\nu_i$ is drawn from $\text{Beta}(s_{i}+1,f_{i}+1)$ in each epoch.
Similar to CUCB-$\kappa$, we use superscripts when explicitly referring to the counters and estimates that CTS uses in epoch $t$. For instance, $\nu_i^t$ denotes a sample drawn from $\text{Beta}(s_{i}^t+1,f_{i}^t+1)$, where $s_{i}^t$ and $f_{i}^t$ are the values of the counters $s_{i}$ and $f_{i}$ in epoch $t$. 

Initially, CTS sets $s_{i}=0$ and $f_{i}=0$. In each epoch $t \geq 1$, it takes a sample $\nu_i$ from the distribution $\text{Beta}(s_{i}+1,f_{i}+1)$ for each arm $i \in \{1,\ldots,m\}$, which is used as an estimate for $\mu_i$. Then, it calls an $(\alpha,\beta)$-approximation algorithm, which takes as input the estimates $\bs{\nu}$ for the expectation vector $\bs{\mu}$ and chooses an action $S_t$. The action $S_t$ depends on the randomness of the approximation algorithm itself in addition to $\bs{\nu}$. After playing the action $S_t$, $\{ X^{(t)}_i \}_{i \in \tau_t}$ is revealed, and the learner collects the reward $R$ just as in CUCB-$\kappa$. Then, CTS updates its counters $s_i$ and $f_i$ for all arms $i \in \tau_t$. 
If $X^{(t)}_i = 1$, then $s_i$ is incremented by one. Otherwise, if $X^{(t)}_i = 0$, then $f_i$ is incremented by one. The counters of the arms that are not in $\tau_t$ remain unchanged.

\section{Regret Analysis}\label{sec:regret}

In this section we analyze the regrets of CUCB-$\kappa$ and CTS. Before delving into the details of the regret analysis, we first prove a key theorem, which shows that the event that the number of times an arm is played by the end of epoch $t$ is less than a linear function of $t$ for some arm has a very low probability for $t$ sufficiently large.
\begin{thm} \label{thm:plays}
	For any learning algorithm, $\eta \in (0,1)$ and for all natural numbers $t \geq t' := 4c^2 / e^2$, where $c := 1/(p^*(1-\eta))^2$, we have
	\begin{align*}
	\Pr \left(\bigcup_{i\in \{1,\ldots,m\}}
	\left\{ T_i^{t+1}\leq \eta p^* t \right\} \right)\leq \frac{m}{t^2} .
	\end{align*}    
\end{thm}

Theorem \ref{thm:plays} is the crux of achieving the theoretical results in this paper since it guarantees that any algorithm obtains sufficiently many observations from each arm, including algorithms that do not explicitly explore any of the arms. This result is very intuitive because when all of the ATPs are positive, the learner observes the states of all arms with positive probability for any action it selects. This will allow us to prove that the estimated expectation vector converges to the true expectation vector independent of the learning algorithm that is used. This fact will be used in proving that the gap-dependent regrets of CUCB-$\kappa$ and CTS are bounded. 

Before continuing the regret analysis, we provide some additional notation. Let $n_{B}$ denote the number of actions whose expected rewards are smaller than $\alpha r_{\bs{\mu}}^*$. These actions are called {\em bad actions}. We re-index the bad actions in increasing order such that $S_{B,l}$ denotes the bad action with $l$th smallest expected reward. The set of bad actions is denoted by ${\cal S}_B := \{S_{B,1},S_{B,2},\ldots,S_{B,n_B}\}$.
Let $\nabla_l := \alpha r_{\bs{\mu}}^*-	r_{\bs{\mu}}(S_{B,l})$ for each $l \in \{1,\ldots,n_B\}$ and $\nabla_{n_B+1}=0$. Accordingly, we let $\nabla_{\max}:=\nabla_1$, $\nabla_{\min}:=\nabla_{n_B}$. We also let $\text{gap}(S_t):=\alpha r_{\bs{\mu}}^*-	r_{\bs{\mu}}(S_{t})$.

\subsection{Regret Analysis for CUCB-$\kappa$}\label{sec:CUCBregret}

First, we show that, given any constant $\delta >0$,  the probability that 
\begin{align*}
\Delta_t^\kappa:=\max_{i \in \{1,\ldots,m\}} |\mu_i-\overline{\mu}_i^{\kappa,t}| < \delta
\end{align*} 
is high, when $t$ is sufficiently large.
This measures how well CUCB-$\kappa$ learns the expected state of each arm by the beginning of epoch $t$, and is directly related to Theorem \ref{thm:plays} as it is related to the number of times each arm is observed by epoch $t$.

\begin{thm}  \label{thm:delta}
	Consider CUCB-$\kappa$, where $\kappa> 0$. For any $\delta >0$ and $\eta \in (0,1)$, let $c := 1/(p^*(1-\eta))^2$, $c_0 := 6\kappa^2/(\delta^2 p^* \eta)$ and 
	$t_1 := \max\{4c^2 / e^2,4c_0^2 / e^2\}$. 
	When CUCB-$\kappa$ is run, we have for all integers $t \geq t_1$
	\begin{align*}
	\Pr(\Delta_{t+1}^\kappa\geq \delta)&\leq \frac{2m}{t^2(1-e^{-\delta^2/2})} +2m e^{-\delta^2\eta p^* t/2}+\frac{m}{t^2} .
	\end{align*}
	
	Consider CUCB-$0$. For any $\delta >0$ and $\eta \in (0,1)$, let $c := 1/(p^*(1-\eta))^2$ and $t' := 4c^2/e^2$.
	When CUCB-$0$ is run, we have for all integers $t\geq t'$
	\begin{align*}
	\Pr(\Delta_{t+1}^0\geq \delta)&\leq \frac{2m}{t^2(1-e^{-2\delta^2})} +2m e^{-2\delta^2\eta p^* t} .
	\end{align*} 
\end{thm}

The upper bound for CUCB-$\kappa$, $\kappa>0$ is looser than the upper bound for CUCB-$0$ given in Theorem \ref{thm:delta}, because of the fact that $t_1 \geq t'$ and additional $m/t^2$ term that appears in the upper bound for  CUCB-$\kappa$, $\kappa>0$. These terms appear as an artifact of the presence of the additional inflation term $\kappa\sqrt{\frac{3\ln t}{2T_i}}$ that appears in the UCB for the expectation vector. While this observation about the upper bound is not sufficient to conclude that CUCB-$\kappa$, $\kappa>0$ is worse than CUCB-$0$ in the setting that we consider, our empirical finding in Section \ref{sec:illustrative} shows that CUCB-$0$ incurs smaller regret than CUCB-$\kappa$, $\kappa>0$ for the movie recommendation application that we consider. 

The next theorem shows that the regret of CUCB-$\kappa$ is bounded for any $T>0$.
\begin{thm} \label{thm:finite}
	The regret of CUCB-$\kappa$, $\kappa>0$ is bounded, i.e., $\forall T \geq 1$
	\begin{align}
	\text{Reg}^{\text{CUCB-}\kappa}_{\bs{\mu},\alpha,\beta}(T) &\leq  \nabla_{\max} \inf_{\eta \in (0,1)} \bigg(\lceil t_1 \rceil + \frac{m\pi^2}{3} \left(\frac{2}{\delta^2}+\frac{3}{2} \right) + 2m \left(1+\frac{2}{\delta^2\eta p^*} \right)\bigg) \label{eqn:bound1}
	\end{align}
	where $\delta := f^{-1}(\nabla_{\min}/2 )$,  $t_1 := \max\{4c^2/e^2,4c_0^2/e^2\}$, $c := 1/(p^*(1-\eta))^2$ and $c_0 := 6\kappa^2/(\delta^2 \eta p^*)$.
	
	The regret of CUCB-$0$ is bounded, i.e., $\forall T \geq 1$
	\begin{align}
	\text{Reg}_{\bs{\mu},\alpha,\beta}^{\text{CUCB-}0}(T)&\leq  \nabla_{\max} \inf_{\eta \in (0,1) } \bigg(\lceil t' \rceil + \frac{m\pi^2}{3}
	\left(1+\frac{1}{2\delta^2} \right) + 
	2m \left(1+\frac{1}{2\delta^2\eta p^*}\right)\bigg) \label{eqn:bound2}
	\end{align}
	where $\delta := f^{-1}(\nabla_{\min}/2 )$, $t' := 4c^2/e^2$ and $c := 1/(p^*(1-\eta))^2$.
\end{thm}

This result is different from the prior results \citep{Chen2016a,Wang2017,Kveton2015c} where $O(\log T)$ gap-dependent regret upper bounds are proven for the CMAB problem. The main difference of our problem from these works is that we assume the minimum ATP to be positive. This allows us to prove the result in Theorem \ref{thm:plays}, by ensuring that each arm is triggered sufficiently many times independent of the exploration strategy used by the learner.

When $\nabla_{\min}$ becomes too small in Theorem \ref{thm:finite}, the regret approaches to infinity because when $\nabla_{\min}$ is too small, $\delta$ will also become very small. For a large class of problems where the bounded smoothess function is $f(x)=\gamma x^{w}$ where $w \in (0,1]$, $\gamma>0$, and $\kappa=1$, the worst case regret is shown to be $O(T^{1-\omega/2}(\ln T)^{\omega/2})$ in \citet{Chen2016a} for a more general setting, hence, this regret bound also holds for our problem. 
We show in Theorem \ref{thm:gap-ind} that for $\kappa=0$ (when CUCB-$\kappa$ becomes the greedy policy), the worst-case regret is bounded by $O(T^{1-w/2})$. 
Note that for the OIM problem and the recommendation problem we consider in the experiments, we have $\omega=1$, and hence, this bound becomes $O(\sqrt{T})$. To prove this, we investigate the behavior of $\text{gap}(S_t)=\alpha r_{\bs{\mu}}^*-	r_{\bs{\mu}}(S_{t})$ based on the change in $\Delta_t^0$. For this, we use Theorem \ref{thm:delta} to bound the expected value of $\Delta_t^0$, which allows us to bound the gap-independent regret.
\begin{thm}\label{thm:gap-ind}
	When the bounded-smoothness function in Assumption 1 is $f(x)=\gamma x^w$ where $\gamma >0$ and $\omega \in (0,1]$, 
	the gap-independent regret bound for CUCB-$0$ is
	\begin{align*}
	Reg_{\bs{\mu},\alpha,\beta}^{\text{CUCB-}0}(T) &\leq \inf_{\eta \in (0,1)} \Bigg(\lceil t' \rceil \nabla_{\text{max}} +  \gamma(2m)^\omega\Bigg[ 2^\omega\Big(\frac{\pi}{2\eta p^*}\Big)^{\omega/2} +3^\omega\Bigg] \frac{T^{1-w/2}}{1-w/2}\Bigg)
	\end{align*}
	where $t' := 4c^2/e^2$ and $c := 1/(p^*(1-\eta))^2$. Hence, the gap-independent regret of CUCB-$0$ is $O(T^{1-\frac{w}{2}})$.
\end{thm}

As a remark, note that the gap-independent regret bound holds for all problem instances where the minimum ATP is at least $p^*$. Essentially, gap-independent means that the regret bound does not depend on $\nabla_{\min}$, and hence $\delta$. Also, $\nabla_{\max}$ can be bounded by the maximum reward, since the reward is assumed to be finite.

\subsection{Regret Analysis for CTS}\label{sec:CTSregret}

	The analysis of the regret of CTS is similar to the regret analysis for CUCB-$\kappa$. We first show that the probability that $\Delta_{t}^{\bs{\nu}}:=\max_i |\nu_i^{t} - \mu_i|$ is greater than  some constant $\delta>0$ becomes smaller as $t$ increases.
	\begin{thm}  \label{thm:deltathompson}
		When CTS is run, for any $\delta>0$ and $\eta \in (0,1)$, we have for all integers $t \geq t' := 4c^2 / e^2$ 
		\begin{align*}
		\Pr(\Delta_{t+1}^{\bs{\nu}} \geq \delta)
		&\leq (3+e^{2\delta})m \Big(\frac{1}{t^2(1-e^{-\delta^2/2})}+e^{-\delta^2\eta p^* t/2}\Big)
		\end{align*} 
		where $c=1/(p^*(1-\eta))^2$.
	\end{thm}
	\cem{The theorem above does not have the extra $m/t^2$ term that Theorem 2 has.}

	The next theorem proves that the regret of CTS is bounded for any $T>0$.
	\begin{thm} \label{thm:finite2}
		The regret of CTS is bounded, i.e., $\forall T \geq 1$
		\begin{align*}
		&\text{Reg}^{\text{CTS}}_{\bs{\mu},\alpha,\beta}(T) \\
		&\leq \nabla_{\max} 
		\inf_{\eta \in (0,1)} \left(\lceil t' \rceil + \frac{(3+e^{2\delta})m\pi^2}{6} \left(1+\frac{2}{\delta^2}\right) + (3+e^{2\delta}) m \left(1+\frac{2}{\delta^2\eta p^*}\right) \right).
		\end{align*}
		where $\delta := f^{-1}(\nabla_{\min}/2 )$,  $t' := 4c^2/e^2$ and $c := 1/(p^*(1-\eta))^2$.
	\end{thm}

The next theorem gives the gap-independent regret bound for CTS.

\begin{thm}\label{thm:gap-ind2}
	When the bounded-smoothness function in Assumption 1 is $f(x)=\gamma x^w$ where $\gamma >0$ and $\omega \in (0,1]$, the gap-independent regret bound for CTS is 
	\begin{align*}
	\text{Reg}_{\bs{\mu},\alpha,\beta}^{\text{CTS}}(T) &\leq \inf_{\eta \in (0,1)} \Bigg(\lceil t' \rceil \nabla_{\max} +   \gamma\Big(2m(3+e^2)\Big)^\omega\Bigg[ \Big(\frac{2\pi}{\eta p^*}\Big)^{\omega/2} +3^\omega \Bigg]\frac{T^{1-w/2}}{1-w/2}\Bigg)
	\end{align*}
	where $t' := 4c^2/e^2$ and $c := 1/(p^*(1-\eta))^2$. Hence, the gap-independent regret of CTS is $O(T^{1-\frac{w}{2}})$.
\end{thm}
Similar to the previous case, for the OIM problem, the bound in Theorem \ref{thm:gap-ind2} becomes $O(\sqrt{T})$.

\section{Illustrative Results}\label{sec:illustrative}
In this section, we evaluate the performance of CUCB-$\kappa$ and CTS on a recommendation problem. This problem has become popular among researchers with the popularization of on-demand media streaming services like Netflix. We use the \textit{MovieLens} dataset for our experiments. The dataset contains 138k people who assigned 20M ratings to 27k movies between January 1995 and March 2015. We use the portion of the dataset that was collected between March 2014 and March 2015, which consists of 750k ratings. For our experiments, we choose 200 movies in total among the movies that were rated more than 200 times: 50 movies with the smallest ratings, 50 movies with the highest ratings, and 100 movies randomly.

\subsection{Definition of the Recommendation Problem}

The problem consists of a weighted bipartite graph $G=(L,R,E,p)$ where $L$ denotes the set of movies, $R$ denotes the set of users,\footnote{Each user corresponds to a pool of individuals with same type of preferences over genres.} $E$ denotes the set of edges between the users, and $p = \{ p_{i,j} \}_{(i,j) \in E}$, where $p_{i,j}$ is the weight of edge $(i,j)$, which corresponds to the probability that movie $i$ influences (attracts) user $j$. 
The goal of the learner is to find a set $S \subseteq L$ of size $k$ that maximizes the expected number of attracted nodes in $R$. 
This problem is an instance of the probabilistic maximum coverage problem \citep{Chen2016a}. Our problem extends this problem by allowing the nodes in $S$ to trigger any $(i,j) \in E $ probabilistically. For instance, this can happen if the users also interact with each other in a social network, where the recommendation made to a user in the network may influence other users into watching the recommended movie via the word of mouth effect. Moreover, both the triggering and influence probabilities are initially unknown. 
We let $p_S^{i,j}$ denote the probability that action $S$ triggers edge $(i,j) \in E $. 
The expected reward is defined as the expected total number of users that are attracted by at least one movie, and is given as
$r_{G}(S) = \sum_{j \in R} (1-\prod_{(i,j) \in E} (1-p_S^{i,j} p_{i,j}) )$.
We assume that $p_S^{i,j}=1$ for the outgoing edges of nodes $i \in S$. This assumption merely says that user $j$ will watch movie $i$ with probability $p_{i,j}$ when the movie is recommended to the user by the system. For the nodes $i \notin S$, $p_S^{i,j} < 1$. For these nodes, $p_S^{i,j}$ denotes the probability that user $i$ gets to know about movie $j$ by the word of mouth affect, without the recommender system showing the movie to the user. For simulations, we set $p_S^{i,j} = p^*$, and evaluate the effect of different values of $p^*$ on the performance of CUCB-$\kappa$ and CTS.

The above problem can be viewed as an instance of CMAB with PTAs. Each edge $(i,j) \in E$ is an arm and the state of each arm is a Bernoulli random variable with success probability $p_{i,j}$. For this problem, Assumption 1 is satisfied with the bounded-smoothness function $f(x)=|E|x$. The monotonicity assumption is also satisfied for this problem since increasing the $p_{i,j}$'s will definitely increase the expected reward. In addition, the reward function is submodular, and hence, it can be shown that using the greedy algorithm in \citet{Nemhauser1978}, we can achieve $(1-1/e)$-approximation to the optimal reward. Hence, the greedy algorithm can be used as a $(1-1/e,1)$-approximation algorithm.

\subsection{Calculation of the Influence Probabilities}

The MovieLens dataset contains the following attributes for each user: UserId, MovieId, Rating, TimeStamp, Title, and the Genre. Hence, we have the rating each user assigned to a movie with a particular genre and title. The dataset contains 20 genres. For each user $j \in R$ we first calculate the user preference vector $\mathbf{u_j}$, which is a unit vector, where each element of the vector corresponds to a coefficient representing how much the user likes a particular genre. 
We assume that the genre distribution of the movies that the users rated represents their genre preferences. Note that a movie can have multiple genres. We also create a $20$ dimensional vector $\mathbf{g_i}$ for each movie $i$, and let $\mathbf{g_{i}}_k=1$ if a movie belongs to genre $k$ and $0$ otherwise. Using this vector, we calculate the genre preference vector $\mathbf{u_j}=\frac{\sum_{i\in L} \mathbf{g_i}+\epsilon_{i,j}}{|| \sum_{i\in L} \mathbf{g_i}+\epsilon_{i,j}||}$ for each user $j\in R$, where $\epsilon_{i,j} \sim \text{Half-Normal}(\sigma=0.05)$. The role of $\epsilon_{i,j}$ here is to account for the fact that the user may possibly explore new genres.
Similarly, for each movie $i \in L$, we calculate the unit movie genre vector $\mathbf{m_{i}}$ as $\mathbf{g_i}/||\mathbf{g_i}||$. Using these, the influence probabilities are calculated as $p_{i,j}= sc \times \frac{<\mathbf{m_{i}},\mathbf{u_j}>r_i}{\max{r_i}}$, $(i,j) \in E$, where $r_i$ is the average rating given by all users to the movie $i$ and $sc$ is a scale factor in $(0,1]$. This way, we took into account the quality in addition to the type (genre) of the movies in determining the influence probabilities.

\subsection{Results}

\begin{figure*}[h!]
	\begin{subfigure}{\textwidth}
		\begin{adjustbox}{width=\textwidth}
			\begin{subfigure}{0.5\columnwidth}
				\centering
				\includegraphics[width=\columnwidth]{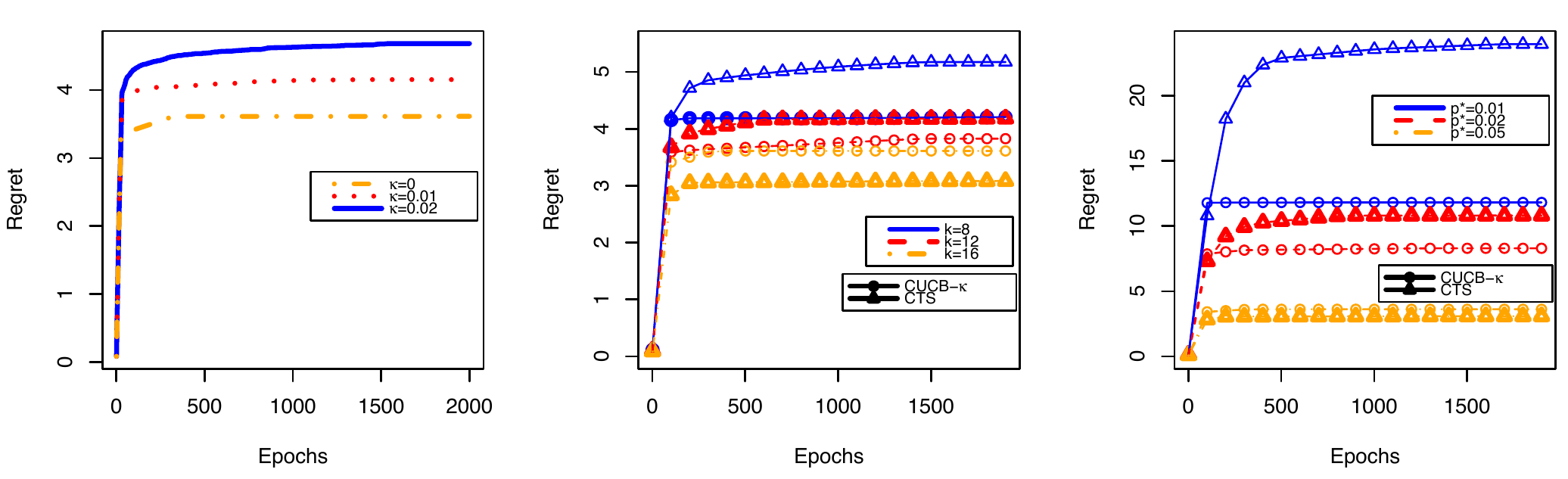}
			\end{subfigure}%
			
		\end{adjustbox}
	\end{subfigure}
	\caption{Regrets of CUCB-$\kappa$ and CTS for different parameter values. Left figure: CUCB-$\kappa$ for $\kappa = 0, 0.01, 0.02$. Middle figure: CUCB-$\kappa$ for $\kappa=0$ (circle marker) and CTS (triangle marker). Right figure: CUCB-$\kappa$ for $\kappa=0$ (circle marker) and CTS (triangle marker).
	}
	\label{CUCB}
\end{figure*} 

All of the presented results are for $p^*=0.05$, $k=16$, $sc=0.2$, and $\kappa=0$ unless otherwise stated. In addition, to be able to make plausible comparisons between settings with different parameters, we consider a scaled version of the regret, where the regret is divided by the $\alpha\beta$ fraction of the optimal reward. $\alpha\beta$ fraction of the optimal reward is calculated by running the $(\alpha,\beta)$-approximation algorithm, which is the greedy algorithm from \citet{Nemhauser1978}, by giving the true influence probabilities as input. We observe from Fig. \ref{CUCB} that the regret is bounded for different values of $k$ and $p^*$ for both of the algorithms and that the regret is bounded for different values of $\kappa$ for CUCB-$\kappa$. It is observed that both CUCB-$\kappa$ and CTS incurs almost no regret after the first 300 epochs. Moreover, for both of the algorithms, as $p^*$ or $k$ increases, the regret becomes smaller. On the other hand, for CUCB-$\kappa$, the regret becomes larger as $\kappa$ increases, which shows that exploration hurts rather than it helps in this setting.
Another observation is that although the regret of CTS and CUCB-$0$ increases as $p^*$ or $k$ decreases, CTS is affected more than CUCB-$0$ by the changes in $p^*$ or $k$. 
\section{Conclusion}\label{sec:conclusion}

In this paper we consider the CMAB problem with positive ATPs, and prove that CUCB-$\kappa$ and CTS achieve bounded gap-dependent regret for any number of epochs $T$. In addition, we prove that CTS and CUCB-$0$ incur at most $O(\sqrt{T})$ gap-independent regret. We also show numerically that CUCB-$\kappa$  and CTS achieve bounded regret in a real-world movie recommendation problem. These results suggest that exploration strategies may not be necessary for learning algorithms that work in CMAB with PTAs where ATPs are positive.

\section{Appendix}

\subsection{Preliminaries}

First, we define the instantaneous $(\alpha,\beta)$-approximation regret, which will be used throughout the analysis.
\begin{defn}
	The instantaneous $(\alpha,\beta)$-approximation regret of algorithm $\pi$ at epoch $t$ is given as
	\begin{align*}
	\text{IR}_{\bs{\mu},\alpha,\beta}^\pi(t)= \alpha\beta r_{\bs{\mu}}^* - r_{\bs{\mu}}(S_t) .
	\end{align*}	
\end{defn}
\cem{I changed this definition to make the total regret equal to sum of the instantaneous regrets. This does not cause any problem, right?}

Next, we define the {\em nice event} at epoch $t$, denoted by
${\cal N}_t$.
\begin{defn}
	${\cal N}_t$ is the event when the $(\alpha,\beta)$-approximation algorithm yields a reward greater than or equal to $\alpha\beta r_{\bs{\mu}}^*$ in expectation, i.e., $\mathbb{E}[r_{\bs{\mu}}(S_t)|{\cal N}_t]\geq \alpha\beta r_{\bs{\mu}}^*$ .
\end{defn}

Let $\hat{\mu}^{(\alpha,\beta),t}_i$ be the estimate of the expected state of arm $i$ given to the $(\alpha,\beta)$-approximation algorithm in epoch $t$ and $\hat{\bs{\mu}}^{(\alpha,\beta),t} := (\hat{\mu}^{(\alpha,\beta),t}_1, \ldots, \hat{\mu}^{(\alpha,\beta),t}_m)$ . For instance, for CUCB-$\kappa$ and CTS, 
$\hat{\mu}^{(\alpha,\beta),t}_i = \bar{\mu}^{\kappa,t}_i$ and 
$\hat{\mu}^{(\alpha,\beta),t}_i = {\nu}^{\kappa,t}_i$, respectively.
Let $\Delta_t^{(\alpha,\beta)}:= \max_{i \in \{1,\ldots,m\}} |\mu_i - \hat{\mu}^{(\alpha,\beta),t}_i|$. The next lemma provides a connection between $\Delta_t^{(\alpha,\beta)}$ and the performance of the $(\alpha,\beta)$-approximation algorithm.
\begin{lemma} \label{lemma:gap}
Given $\Delta_t^{(\alpha,\beta)}<f^{-1}(\nabla_{\min}/2 )$, an $(\alpha,\beta)$-approximation algorithm will select an action for which $\text{gap}(S_t) \leq 0$, with probability at least $\beta$.
\end{lemma}

\begin{proof}
	Let $\hat{S}^*_t := \argmax_{S \in {\cal S}} r_{ \hat{\bs{\mu}}^{(\alpha,\beta),t} }(S)$ be the optimal action given the  estimated expectation vector $\hat{\bs{\mu}}^{(\alpha,\beta),t}$. Using Assumption 1 and the fact that $r_{\hat{\bs{\mu}}^{(\alpha,\beta),t}}(\hat{S}^*_t) \geq r_{\hat{\bs{\mu}}^{(\alpha,\beta),t}}(S^*)$, we get
	\begin{align}
		r_{\bs{\mu}}(S_t) &\geq 	r_{\hat{\bs{\mu}}^{(\alpha,\beta),t}}(S_t)-f(\Delta_t^{(\alpha,\beta)}) \notag \\ 
		&\geq \alpha
			r_{\hat{\bs{\mu}}^{(\alpha,\beta),t}}(\hat{S}^*_t)-f(\Delta_t^{(\alpha,\beta)}) \notag \\ &\geq \alpha
			 r_{\hat{\bs{\mu}}^{(\alpha,\beta),t}}(S^*)-f(\Delta_t^{(\alpha,\beta)}) \notag \\ 
			 &\geq \alpha	r_{\bs{\mu}}^*-2f(\Delta_t^{(\alpha,\beta)}) \label{eqn:lemma1}
	\end{align}
	with probability at least $\beta$.
	Next, we show that \eqref{eqn:lemma1} and 
	$S_t \in \{S_{B,1}, \ldots, S_{B,n_B}\}$ cannot hold at the same time.  
	$S_t \in \{S_{B,1}, \ldots, S_{B,n_B}\}$ and $\{ r_{\bs{\mu}}(S_t) \geq  \alpha r_{\bs{\mu}}^*-2f(\Delta_t^{(\alpha,\beta)}) \}$ implies that 
	\begin{align*}
	\alpha r_{\bs{\mu}}^*-\nabla_{\min} \geq r_{\bs{\mu}}(S_t) \geq \alpha r_{\bs{\mu}}^*-2f(\Delta_t^{(\alpha,\beta)}) .
	\end{align*}
	The above set of inequalities cannot hold when $\Delta_t^{(\alpha,\beta)}<f^{-1}(\nabla_{\min}/2 )$, since in this case we have $\nabla_{\min} > 2f(\Delta_t^{(\alpha,\beta)})$.
	This implies that $S_t \notin \{S_{B,1}, \ldots, S_{B,n_B}\}$ with probability at least $\beta$ when  $\Delta_t^{(\alpha,\beta)}<f^{-1}(\nabla_{\min}/2 )$. 
\end{proof}
\subsection{Proof of Theorem 1}
The following lemma shows that the number of times an arm is triggered increases linearly in $t$ with probability at least $1-1/t^2$.

\begin{lemma} \label{lemma:plays2}
     For any learning algorithm and for any $\eta \in (0,1)$, we have
     \begin{align*}
     \Pr(T_i^{t+1}\leq \eta p^* t)\leq \frac{1}{t^2} 
     \end{align*}
     for all integers $t \geq t' := 4c^2/e^2$, where $c=1/(p^*(1-\eta))^2$.
\end{lemma}
\begin{proof}
	In the proof, we follow a procedure similar to the proof of Theorem 3 in \citet{Nima}. Let $\mathbf{1}_{\{i \in \tau_j\}}$ be the indicator variable, which is $1$ if arm $i$ is triggered in epoch $j$ and $0$ otherwise. Recalling that $\sum_{j=1}^{t} \mathbf{1}_{\{i \in \tau_j\}} = T_i^{t+1}$ and using the Hoeffding's inequality, we obtain
	\begin{align*}
	\Pr(T_i^{t+1}-\mathbb{E}[T_i^{t+1}]\leq -z )\leq e^{-2z^2/t} 
	\end{align*}
	for $z>0$.
	By setting $z=\sqrt{t\ln t}$, we obtain
	\begin{align*}
	\Pr(T_i^{t+1}-\mathbb{E}[T_i^{t+1}]\leq -\sqrt{t\ln t})\leq e^{-2\ln t }=\frac{1}{t^2} .
	\end{align*}
	
	Note that $\Pr(i \in \tau_t)=\mathbb{E}[\mathbf{1}_{\{i \in \tau_t\}}]\geq p^*$ for any $t\in \mathbb{N}_{+}$, and hence, $\mathbb{E}[T_i^{t+1}]\geq p^*t$. Following this observation, we obtain that $\{T_i^{t+1}-p^*t\leq -\sqrt{t\ln t})\} \subseteq \{T_i^{t+1}-\mathbb{E}[T_i^{t+1}]\leq -\sqrt{t\ln t})\}$. Hence, we have 
	\begin{align*}
	\Pr(T_i^{t+1}-p^*t\leq -\sqrt{t\ln t})\leq \frac{1}{t^2} . \\
	\end{align*}
	
	We want to show that, given any $\eta \in (0,1)$, for $t$ sufficiently large, the inequality $p^* t - \sqrt{t\ln t}\geq \eta p^* t$ holds. Let $h(t) := p^{*} t(1-\eta)$ and $g(t) := \sqrt{t\ln t}$. Let $ t^{+} \in \mathbb{R}$ be the greatest root of the equation $h(t)=g(t)$, i.e., for all $t>t^{+}$, $h(t)>g(t)$. To see this, we observe that $h(1) > g(1)$, and
	$h'(t)=p^* (1-\eta)$ and $g'(t)=\frac{1}{2}[\sqrt{\frac{\ln t}{t}}+\sqrt{\frac{1}{t \ln t}}]$.
	Therefore $\lim_{t\rightarrow \infty}g(t)=0$ whereas $h'(t)=p^* (1-\eta) \in \mathbb{R^+}$ for all $t \in \mathbb{N}_{+}$. As a result, $\exists t'' \in \mathbb{R}$ such that $h(t)> g(t)$, $\forall t\geq t''$. This justifies that $t^{+}$ finite.

Next, we find an upper bound for $t^{+}$. Note that $h(t)=g(t)$ implies $t(p^*(1-\eta))^2=\ln t$. Let $c=1/(p^*(1-\eta))^2$. Thus, taking exponentials of the both sides, we want to solve the equation $e^{t/c}=t$ whose roots are given by the equation
	\begin{align*}
		t=\frac{c}{\ln e}\text{glog} \left(\frac{c}{\ln e} \right)=c\text{glog}(c).
	\end{align*} 
Then, using the bound given in \citet{Kalman2001}, we obtain\cem{Follows from (6) in Kalman.}
	\begin{align*}
		t^+ \leq c^2 \left( \frac{k}{e} \right)^{\frac{k}{k-1}} ~\text{for} ~ k>1.
	\end{align*}
	By setting $k=2$, we obtain the following upper bound on $ t^{+} $:
	\begin{align} \label{eq:glog}
		t^+ \leq \frac{4c^2}{e^2}.
	\end{align}
	
	This implies that $\forall t \geq 4c^2/e^2 \geq t^+$, we have $h(t)\geq g(t)$, which further implies that $p^* t - \sqrt{t\ln t}\geq \eta p^* t$. Thus, for all integers $ t\geq 4c^2/e^2$, we have $\{T_i^{t+1}\leq \eta p^*t\}\subseteq\{T_i^{t+1}\leq p^*t-\sqrt{t\ln t}\}$ and we obtain
	\begin{align*}
	\Pr(T_i^{t+1}\leq \eta p^*t)\leq \frac{1}{t^2}.
	\end{align*}

\end{proof}

Using the result of Lemma \ref{lemma:plays2} and the union bound over all arms, we obtain
\begin{align*}
	\Pr \left(\bigcup_{i \in \{1,\ldots, m\}}\{T_i^{t+1}\leq \eta p^*t\} \right)\leq \frac{m}{t^2}, ~\forall t\geq t' .
\end{align*}

\subsection{Proof of Theorem 2}
First, we show that the inflation factor in the UCB estimate is less than some constant $\epsilon>0$ with probability greater than $1-1/t^2$.
\begin{lemma}\label{lemma:UCB}
Given $\eta \in (0,1)$ $\kappa>0$ and $\epsilon>0$, for all integers $t\geq \max\{t'_0:= 4c_0^{2}/e^2,4c^2/e^2\}$ where $c_0 := 3\kappa^2/(2\epsilon^2 p^* \eta)$ and $c := 1/(p^* (1-\eta))^2$, we have
\begin{align*}
\Pr \left( \kappa\sqrt{\frac{3\ln t}{2T_{i}^{t+1}}}\geq \epsilon \right) 
\leq \frac{1}{t^2} .
\end{align*}
\end{lemma}
\begin{proof}
Note that
\begin{align*}
\kappa \sqrt{\frac{3\ln t}{2T_{i}^{t+1}}}\geq\epsilon &\Longleftrightarrow \kappa^2 \frac{3\ln t}{2T_{i}^{t+1} } \geq\epsilon^2  \Longleftrightarrow T_{i}^{t+1}\leq\frac{3\kappa^2 \ln t}{2\epsilon^2}  \Longleftrightarrow T_{i}^{t+1}\leq\frac{3\kappa^2 \ln t}{2\epsilon^2 p^* t} p^* t.
\end{align*}

Hence, when $\frac{3\kappa^2 \ln t}{2\epsilon^2 p^* t} \leq \eta $, we have

\begin{align*}
\Bigg\{\kappa \sqrt{\frac{3\ln t}{2T_{i}^{t+1}}}\geq\epsilon\Bigg\} \Rightarrow \{T_{i}^{t+1} \leq \eta p^* t\}
\end{align*}

which implies that for $t \geq t' = 4c^2/\epsilon^2$,

\begin{align} \label{eq:infl}
\Pr\Bigg(\kappa \sqrt{\frac{3\ln t}{2T_{i}^{t+1}}} \geq\epsilon\Bigg)\leq \Pr( T_{i}^{t+1} \leq \eta p^* t) \leq \frac{1}{t^2}
\end{align}
by Lemma \ref{lemma:plays2}. Next, following the procedure we use in the proof of Lemma \ref{lemma:plays2}, we show that the inequality $\frac{3\kappa^2 \ln t}{2\epsilon^2 p^* t} \leq \eta $ is satisfied when $t\geq t'_0$, which will imply that the result in \eqref{eq:infl} is satisfied for $t \geq \max\{ t',t'_0 \}$.

	We want to show that, given some $\eta \in (0,1)$, for $t$ sufficiently large, the inequality $\frac{3\kappa^2 \ln t}{2\epsilon^2 p^* t} \leq \eta $ holds. Let $h_0(t) :=  2\epsilon^2 p^* \eta t$ and $g_0(t) := 3\kappa^2 \ln t$. Let $ t_0^{+} \in \mathbb{R}$ be the greatest root of the equation $h_0(t)=g_0(t)$, i.e., for all $t>t_0^{+}$, $h_0(t)>g_0(t)$. To see this, we observe that $h_0(1) > g_0(1)$, and
	$h'_0(t)=2\epsilon^2 p^* \eta$ and $g'_0(t)=3\kappa^2/t$.
	Therefore $\lim_{t\rightarrow \infty} g'_0(t)=0$ whereas $h'_0(t)=2\epsilon^2 p^* \eta \in \mathbb{R^+}$ for all $t \in \mathbb{N}_{+}$. As a result, $\exists t'' \in \mathbb{R}$ such that $h'_0(t)> g'_0(t)$, $\forall t\geq t''$. This justifies that $t_0^{+}$ is finite.\cem{Here, the $'$ terms correspond to derivatives.}
	
Next, we find an upper bound for $t_0^{+}$. Note that $h_0(t)=g_0(t)$ implies that $\ln t=2t\epsilon^2 p^*   \eta/(3\kappa^2) $. Let $c_0=3\kappa^2/(2\epsilon^2 p^*   \eta)$. Thus, taking exponentials of the both sides, we want to solve the equation $e^{t/c_0}=t$ whose roots are given by the equation
	\begin{align*}
		t=\frac{c_0}{\ln e}\text{glog} \left(\frac{c_0}{\ln e} \right)=c_0 \text{glog}(c_0).
	\end{align*} 
Then, using the bound given in \citet{Kalman2001}, we obtain
	\begin{align*}
		t_0^+ \leq c_0^2 \left( \frac{k}{e} \right)^{\frac{k}{k-1}} ~\text{for} ~ k>1.
	\end{align*}
	By setting $k=2$, we obtain the following upper bound on $ t_0^{+} $:
	\begin{align} \label{eq:glog2}
		t_0^+ \leq \frac{4c_0^2}{e^2}.
	\end{align}
	
	This implies that $\forall t \geq 4c_0^2/e^2 \geq t_0^+$, we have $h_0(t)\geq g_0(t)$, which further implies that $\frac{3\kappa^2 \ln t}{2\epsilon^2 p^* t} \leq \eta $. Hence, the inequalities in \eqref{eq:infl} hold for all integers $t\geq \max\{4c_0^{2}/e^2,4c^2/e^2\}$.
\end{proof}

The next lemma shows that the sample mean estimate of the expected state of any arm cannot be too far away from its true value when $t$ is large.

\begin{lemma}\label{lemma:sample}
For any $\delta >0$ and $\eta \in (0,1)$, we have
\begin{align*}
\Pr(|\hat{\mu}_{i}^{t+1}-\mu_{i}|\geq\delta)\leq  \frac{2}{t^2(1-e^{-2\delta^2})} +2 e^{-2\delta^2\eta p^* t}
\end{align*}
for all integers $t \geq t' := 4c^2/e^2$, where $c=1/(p^*(1-\eta))^2$.
\end{lemma}
 \begin{proof}

Given that $T_i^{t+1}=j$, we have $j\hat{\mu}_i^{t+1}=\sum_{k=1}^{t} X_{i}^{(k)}1_{\{i \in \tau_k\}}$. Since $X_{i}^{(k)}1_{\{i \in \tau_k\}} \in [0,1]$, by Hoeffding's inequality, we obtain
\begin{align*}
\Pr(|\mu_i-\hat{\mu}_i^{t+1}|\geq \delta|T_i^{t+1}=j) &=\Pr(|T_i^{t+1}\mu_i-T_i^{t+1}\hat{\mu}_i^{t+1}|\geq T_i^{t+1}  \delta|T_i^{t+1}=j)\\&=\Pr(|\mathbb{E}[T_i^{t+1}\hat{\mu}_i^{t+1}]-T_i^{t+1}\hat{\mu}_i^{t+1}|\geq j \delta|T_i^{t+1}=j)\\&=\Pr(|\mathbb{E}[j\hat{\mu}_i^{t+1}]-j\hat{\mu}_i^{t+1}|\geq j \delta|T_i^{t+1}=j)\\&\leq 2e^{-2(\delta j)^2 / j} = 2e^{-2\delta ^2 j} .
\end{align*}
Then, using the law of total probability, we obtain
\begin{align*}
		\Pr(|\mu_i-\hat{\mu}_i^{t+1}|\geq \delta) &=\sum_{j=0}^{t} \Pr(|\mu_i-\hat{\mu}_i^{t+1}|\geq \delta|T_i^{t+1}=j)\Pr(T_i^{t+1}=j) \\&\leq \sum_{j=0}^{t} 2e^{-2\delta^2 j}\Pr(T_i^{t+1}=j) .
\end{align*}
 Let $t^*=\eta p^* t$. For $0 \leq j < t^*$, we have $\{T_i^{t+1}=j\}\subseteq \{T_i^{t+1}\leq t^{*}\}$; hence, $\Pr(T_i^{t+1}=j)\leq \Pr(T_i^{t+1}\leq t^{*}) \leq  1/t^2$ when $t\geq t'$ by Lemma \ref{lemma:plays2}. We also have $\sum_{j=\lceil t^*\rceil}^{t} \Pr(T_i^{t+1}=j)\leq 1$. Using these, we proceed as follows: For $t \geq t'$, we have
\begin{align*}
	\Pr(|\mu_i-\hat{\mu}_i^{t+1}|\geq \delta) &\leq \sum_{j=0}^{\lfloor t^*\rfloor} 2e^{-2\delta^2 j}\Pr(T_i^{t+1}=j)+\sum_{j=\lceil t^*\rceil}^{t} 2e^{-2\delta^2 j}\Pr(T_i^{t+1}=j) \\& \leq \frac{1}{t^2}\sum_{j=0}^{\lfloor t^*\rfloor} 2e^{-2\delta^2 j}+ 2e^{-2\delta^2 \lceil t^*\rceil }\sum_{j=\lceil t^*\rceil}^{t}\Pr(T_i^{t+1}=j)\\&\leq \frac{2}{t^2(1-e^{-2\delta^2})} +2 e^{-2\delta^2\lceil t^*\rceil} \\&\leq \frac{2}{t^2(1-e^{-2\delta^2})} +2 e^{-2\delta^2\eta p^* t}.
\end{align*}
\end{proof}
Note that $\kappa \sqrt{\frac{3\ln t}{2T_i^{t+1}}}<\delta/2$ and $|\hat{\mu}_{i}^{t+1} - \mu_i|< \delta/2$ for all $i \in \{1,\ldots,m\}$ implies that $\Delta_{t+1}^{\kappa} < \delta$ by the triangle inequality. Hence, by the contrapositive of this proposition, we have $\{\Delta_{t+1}^{\kappa} \geq \delta\} \subseteq \big\{ \bigcup_i\{|\hat{\mu}_{i}^{t+1} - \mu_i|\geq \delta/2\}\big\} \cup \big\{\bigcup_i \{\kappa \sqrt{\frac{3\ln t}{2T_i^{t+1}}}\geq\delta/2\}\big\}$. Using the union bound and the results of Lemmas \ref{lemma:UCB} and \ref{lemma:sample}, we obtain
\begin{align}
	\Pr(\Delta_{t+1}^{\kappa} \geq \delta)  
	&\leq 
	\sum_{i = 1}^m \Pr(|\hat{\mu}_{i}^{t+1} - \mu_i|\geq \delta/2)\notag + \sum_{i = 1}^m \Pr \left(\kappa \sqrt{\frac{3\ln t}{2T_i^{t+1}}}\geq\delta/2 \right) \notag \\
	&\leq  \frac{2m}{t^2(1-e^{-\delta^2/2})} +2m e^{-\delta^2\eta p^* t/2}+\frac{m}{t^2}\notag
	\end{align}	 
for $t\geq t_1$. When $\kappa=0$, the result directly follows from Lemma \ref{lemma:sample}.

\subsection{Proof of Theorem 3}
	Fix any $\eta \in (0,1)$. Note that $\delta >0$ since $\nabla_{\text{min}} >0$. First, we prove the theorem for $\kappa>0$. For $T \leq \lceil t_1 \rceil$, the regret bound is $T \nabla_{\text{max}}$. By Lemma \ref{lemma:gap}, we have
	\begin{align*}
	\{\Delta_t^\kappa < \delta\} &\Rightarrow \Pr(r_{\bs{\mu}}(S_t) \geq \alpha r_{\bs{\mu}}^{*})\geq \beta \\&\Rightarrow \mathbb{E}[r_{\bs{\mu}}(S_t)]\geq \alpha\beta r_{\bs{\mu}}^{*} \Rightarrow {\cal N}_t \\&\Rightarrow \mathbb{E}[\text{IR}_{\bs{\mu},\alpha,\beta}^{\text{CUCB-}\kappa}(t)] \leq 0.
	\end{align*} 
	Hence, we have the following for $T >\lceil t_1 \rceil$:
	\begin{align*}
	&\mathbb{E} \left[\sum_{t=1}^{T}\text{IR}_{\bs{\mu},\alpha,\beta}^{\text{CUCB-}\kappa}(t) \right]  
	=	\mathbb{E} \left[\sum_{t=1}^{\lceil t_1 \rceil}\text{IR}_{\bs{\mu},\alpha,\beta}^{\text{CUCB-}\kappa}(t) \right]+	\mathbb{E}\left[ \sum_{t=\lceil t_1 \rceil+1}^{T}\text{IR}_{\bs{\mu},\alpha,\beta}^{\text{CUCB-}\kappa}(t) \right]  
	\\ &\leq \nabla_{\text{max}}\lceil t_1 \rceil 
	+ \sum_{t=\lceil t_1 \rceil+1}^{T}\mathbb{E}[\text{IR}_{\bs{\mu},\alpha,\beta}^{\text{CUCB-}\kappa}(t)|\Delta_t^\kappa < \delta]\Pr(\Delta_t^\kappa < \delta) \\
	&+ \sum_{t=\lceil t_1 \rceil+1}^{T}\mathbb{E}[\text{IR}_{\bs{\mu},\alpha,\beta}^{\text{CUCB-}\kappa}(t)|\Delta_t^\kappa \geq \delta]\Pr(\Delta_t^\kappa \geq \delta) 
	\\ &= \nabla_{\text{max}}\lceil t_1 \rceil + \sum_{t=\lceil t_1 \rceil+1}^{T}\mathbb{E}[\text{IR}_{\bs{\mu},\alpha,\beta}^{\text{CUCB-}\kappa}(t)|{\cal N}_t]\Pr(\Delta_t^\kappa < \delta) \\
	&+ \sum_{t=\lceil t_1 \rceil+1}^{T}\mathbb{E}[\text{IR}_{\bs{\mu},\alpha,\beta}^{\text{CUCB-}\kappa}(t)|\Delta_t^\kappa \geq \delta]\Pr(\Delta_t^\kappa \geq \delta) \\ &\leq \nabla_{\text{max}}\lceil t_1 \rceil+ \sum_{t=\lceil t_1 \rceil+1}^{T}\mathbb{E}[\text{IR}_{\bs{\mu},\alpha,\beta}^{\text{CUCB-}\kappa}(t)|\Delta_t^\kappa \geq \delta] \\
	&\times \Bigg(\frac{2m}{(t-1)^2(1-e^{-\delta^2/2})} +2m e^{-\delta^2\eta p^*(t-1)/2} + \frac{m}{(t-1)^2}\Bigg)\\ &\leq \nabla_{\text{max}}\lceil t_1 \rceil+ \nabla_{\text{max}}\sum_{t=1}^{\infty} \left(\frac{2m}{t^2(1-e^{-\delta^2/2})}+2m e^{-\delta^2\eta p^* t/2} +\frac{m}{t^2} \right)
	\\&\leq \nabla_{\text{max}} \left(\lceil t_1 \rceil+\frac{m\pi^2}{3(1-e^{-\delta^2/2})} +\frac{2me^{-\delta^2\eta p^*/2}}{1-e^{-\delta^2\eta p^*/2}}+\frac{m\pi^2}{6} \right)
	\\&\leq \nabla_{\text{max}} \left(\lceil t_1 \rceil+\frac{m\pi^2}{3(1-e^{-\delta^2/2})} +\frac{2m}{1-e^{-\delta^2\eta p^*/2}}+\frac{m\pi^2}{6} \right).
	\end{align*}
	
	We know that $\frac{1}{1-e^{-x}}\leq 1+\frac{1}{x}$ for $x>0$. Using this, we obtain
	\begin{align}
		\mathbb{E} \left[\sum_{t=1}^{T}\text{IR}_{\bs{\mu},\alpha,\beta}^{\text{CUCB-}\kappa}(t) \right] \notag 
	&\leq \nabla_{\text{max}} \left(\lceil t_1 \rceil+\frac{m\pi^2}{3(1-e^{-\delta^2/2})} +\frac{2m}{1-e^{-\delta^2\eta p^*/2}} +\frac{m\pi^2}{6} \right) \notag
	\\& \leq \nabla_{\text{max}} \left(\lceil t_1 \rceil+\frac{m\pi^2}{3}  \left(1+\frac{2}{\delta^2}+\frac{1}{2}\right) 
	+ 2m \left(1+\frac{2}{\delta^2\eta p^*} \right) \right). \label{eq:etamin1}
	\end{align}
Since \eqref{eq:etamin1} holds for any $\eta \in (0,1)$, the final result is obtained by minimizing \eqref{eq:etamin1} over $\eta \in (0,1)$.
	
For $\kappa=0$, we follow the same procedure. For $T \leq \lceil t' \rceil$, the regret bound is $T \nabla_{\text{max}}$. By Lemma \ref{lemma:gap}, we have
	\begin{align*}
	\{\Delta_t < \delta\} &\Rightarrow \Pr(r_{\bs{\mu}}(S_t) \geq \alpha r_{\bs{\mu}}^{*})\geq \beta \\&\Rightarrow \mathbb{E}[r_{\bs{\mu}}(S_t)]\geq \alpha\beta r_{\bs{\mu}}^{*} \Rightarrow {\cal N}_t \\&\Rightarrow \mathbb{E}[\text{IR}_{\bs{\mu},\alpha,\beta}^{\text{CUCB-}\kappa}(t)] \leq 0.
	\end{align*} 
	Hence, we have the following for $T >\lceil t' \rceil$:
	\begin{align*}
	&\mathbb{E} \left[\sum_{t=1}^{T}\text{IR}_{\bs{\mu},\alpha,\beta}^{\text{CUCB-}\kappa}(t) \right] =\mathbb{E} \left[\sum_{t=1}^{\lceil t' \rceil}\text{IR}_{\bs{\mu},\alpha,\beta}^{\text{CUCB-}\kappa}(t) \right] + \mathbb{E}\left[ \sum_{t=\lceil t' \rceil+1}^{T}\text{IR}_{\bs{\mu},\alpha,\beta}^{\text{CUCB-}\kappa}(t) \right]  
	\\ &\leq \nabla_{\text{max}}\lceil t' \rceil + \sum_{t=\lceil t' \rceil+1}^{T} \mathbb{E} \left[\text{IR}_{\bs{\mu},\alpha,\beta}^{\text{CUCB-}\kappa}(t)|\Delta_t < \delta \right] \Pr(\Delta_t < \delta) \\
	&+ \sum_{t=\lceil t' \rceil+1}^{T} \mathbb{E} \left[\text{IR}_{\bs{\mu},\alpha,\beta}^{\text{CUCB-}\kappa}(t)|\Delta_t \geq \delta \right] \Pr(\Delta_t \geq \delta) \\ 
	&\leq \nabla_{\text{max}}\lceil t' \rceil+ \sum_{t=\lceil t' \rceil+1}^{T} \mathbb{E} \left[\text{IR}_{\bs{\mu},\alpha,\beta}^{\text{CUCB-}\kappa}(t)|\Delta_t \geq \delta \right] \\
	&\times \left(\frac{2m}{(t-1)^2(1-e^{-2\delta^2})} +2m e^{-2\delta^2\eta p^* (t-1)} \right) \\ 
	&\leq \nabla_{\text{max}}\lceil t' \rceil+ \nabla_{\text{max}}\sum_{t=1}^{\infty}
	\left(\frac{2m}{t^2(1-e^{-2\delta^2})}+2m e^{-2\delta^2\eta p^* t}\right)
	\\&\leq \nabla_{\text{max}} \left(\lceil t' \rceil+\frac{m\pi^2}{3(1-e^{-2\delta^2})} +\frac{2m}{1-e^{-2\delta^2\eta p^*}} \right) .
	\end{align*}
	
	We know that $\frac{1}{1-e^{-x}}\leq 1+\frac{1}{x}$ for $x>0$. Using this, we obtain
	\begin{align}
	\mathbb{E} \left[\sum_{t=1}^{T}\text{IR}_{\bs{\mu},\alpha,\beta}^{\text{CUCB-}\kappa}(t) \right]  
	& \leq \nabla_{\text{max}} \left( \lceil t' \rceil + \frac{m\pi^2}{3(1-e^{-2\delta^2})} +\frac{2m}{1-e^{-2\delta^2\eta p^*}} \right) \notag \\ 
	& \leq \nabla_{\text{max}} \left(\lceil t' \rceil + \frac{m\pi^2}{3} \left(1+\frac{1}{2\delta^2}\right) 
	+ 2m \left( 1+\frac{1}{2\delta^2\eta p^*} \right) \right). \label{eq:etamin2}
	\end{align}

Finally, we minimize \eqref{eq:etamin2} over $\eta \in (0,1)$ via the procedure we follow for minimizing \eqref{eq:etamin1}.

\subsection{Proof of Theorem 4}
For $T \leq \lceil t' \rceil$, the regret bound is $T \nabla_{\text{max}}$. To establish a regret bound for $T > \lceil t' \rceil$,
first, we show that the instantaneous regret $\text{IR}_{\bs{\mu},\alpha,\beta}^{\text{CUCB-}\kappa}(t)$ can be bounded by using $\Delta_{t}^{0}$ and the bounded smoothness function $f$.
\begin{lemma} \label{lemma:al}
For any $t\in \mathbb{N}_+$ and expectation vector $\bs{\mu}$, the instantaneous regret of CUCB-$\kappa$ is bounded as follows:
\begin{align*}
\mathbb{E}[\text{IR}_{\bs{\mu},\alpha,\beta}^{\text{CUCB-}\kappa}(t)]\leq 2\mathbb{E}[f(\Delta_{t}^{\kappa})].
\end{align*}

\end{lemma}
\begin{proof}
Taking expectation of both sides of the equation \eqref{eqn:lemma1} in Lemma \ref{lemma:gap}, we obtain the desired result.
\end{proof}
In order to make use of Lemma \ref{lemma:al} and find an upper bound for the instantaneous regret of CUCB-$0$, we bound the value $\mathbb{E}[\Delta_{t+1}^{0}]$ for $t>\lceil t' \rceil$.

\begin{lemma}\label{UCBplusSample}
For any $\eta \in (0,1)$, we have	
\begin{align*}
\mathbb{E}[\Delta_{t+1}^{0}] \leq t^{-1/2}
\left(2m\sqrt{\frac{\pi}{2\eta p^*}}+3m\right)
\end{align*}
for all integers $t \geq t' := 4c^2/e^2$, where $c := 1/(p^*(1-\eta))^2$.
\end{lemma}
\begin{proof}
We have
\begin{align*}
\mathbb{E}[\Delta_{t+1}^{0}] = \int_{0}^{1} \Pr(\Delta_{t+1}^{0} \geq \delta)\text{d}\delta.
\end{align*}
Hence, by Theorem 2, since $\Delta_{t+1}^{0} \in [0,1]$, we have for any $\theta >0$
\begin{align*}
\mathbb{E}[\Delta_{t+1}^{0}]  &\leq \int_{0}^{t^{-\theta}} \Pr(\Delta_{t+1}^{0} \geq \delta)\text{d}\delta+\int_{t^{-\theta}}^{1} \Big(\frac{2m}{t^2(1-e^{-2\delta^2})} +2m e^{-2\delta^2\eta p^* t}\Big)\text{d}\delta \\&\leq t^{-\theta} +\int_{t^{-\theta}}^{1} \Big(\frac{2m}{t^2(1-e^{-2\delta^2})} +2m e^{-2\delta^2\eta p^* t}\Big)\text{d}\delta
\end{align*}
for $t>\lceil t' \rceil$.

Note that $e^{-2\delta^2\eta p^* t}$ is just a scaled version of the pdf of the normal distribution. Let $B=1/(2\sqrt{\eta p^*t})$, then we have

\begin{align*}
\int e^{-2\delta^2\eta p^* t}\text{d}\delta = \int e^{-\delta^2/(2B^2)}\text{d}\delta = B\sqrt{2\pi} \int \frac{1}{ B\sqrt{2\pi}} e^{\frac{-\delta^2}{2B^2}}\text{d}\delta.
\end{align*} 
Hence, setting $B=\sqrt{1/(2At)}$ ($A=2\eta p^*$), via the fact that $1/(1-e^{-x})\leq 1 + 1/x$ for all $x>0$, we obtain
\begin{align*} 
\mathbb{E}[\Delta_{t+1}^{0}] &\leq t^{-\theta}+ 2m\sqrt{\frac{\pi}{At}}+ \int_{t^{-\theta}}^{1} \Big(\frac{2m}{t^2(1-e^{-2\delta^2})} \Big)\text{d}\delta \notag \\&\leq t^{-\theta}+ 2m\sqrt{\frac{\pi}{At}}+ \frac{m}{t^2}\int_{t^{-\theta}}^{1} \Big(2+\frac{1}{\delta^2}\Big)\text{d}\delta \notag \\&=t^{-\theta}+2m\sqrt{\frac{\pi}{At}}+mt^{-2}\bigg(2\delta-\frac{1}{\delta}\bigg)\bigg|_{\delta=t^{-\theta}}^{\delta=1} \notag \\&=t^{-\theta}+2m\sqrt{\frac{\pi}{At}}+mt^{-2}\bigg(2-1-2t^{-\theta}+t^{\theta}\bigg)\notag \\&\leq t^{-\theta}+ 2m\sqrt{\frac{\pi}{At}}+ mt^{-2}(t^\theta +1) \notag \\
&\rev{=} t^{-\theta}+ 2m\sqrt{\frac{\pi}{At}}+ mt^{\theta-2}+mt^{-2}.  \notag
\end{align*}
The order of $t$ in this bound is minimized when $\theta=1$; thus, noticing that $2m+1\leq 3m$, we have 

\begin{align} 
\mathbb{E}[\Delta_{t+1}^{0}] 
&\leq 2m\sqrt{\frac{\pi}{At}} + (m+1)t^{-1}+mt^{-2} \notag 
\\&\leq 2mt^{-1/2}\sqrt{\frac{\pi}{A}}+(m+1)t^{-1/2}+mt^{-1/2} \notag 
\\&\leq t^{-1/2}\Big(2m\sqrt{\frac{\pi}{A}}+2m+1\Big) \notag \\  &\leq t^{-1/2}\Big(2m\sqrt{\frac{\pi}{A}}+3m\Big) . \notag 
\end{align} 

\end{proof}
Next, using this lemma (Lemma \ref{UCBplusSample}), we show that the gap-independent regret for CUCB-$\kappa$ (when $\kappa=0$) is at most $O(T^{1-w/2})$ for a large set of problems where the bounded-smoothness function $f(x)$ is $\gamma x^\omega$, for $\omega \in (0,1]$. 

Note that $f$ is concave. Hence, via Lemma \ref{lemma:al}, Lemma \ref{UCBplusSample} and Jensen's inequality, we have
\begin{align*}
\mathbb{E}[\text{IR}_{\bs{\mu},\alpha,\beta}^{\text{CUCB-}0}(t+1)] &\leq 2\mathbb{E}[f(\Delta_{t+1}^{0})] \\& \leq
2f(\mathbb{E}[\Delta_{t+1}^{0}])
\\&\leq 2f\bigg(2m\sqrt{\frac{\pi}{At}} +3m\sqrt{\frac{1}{t}} \bigg) 
\\&\leq f\Big(4m\sqrt{\frac{\pi}{At}}\Big) + f\Big(6m\sqrt{\frac{1}{t}}\Big)  \\&=\gamma\bigg((4m)^\omega \Big(\frac{\pi}{At}\Big)^{\omega/2}+(6m)^{\omega} t^{-\omega/2} \bigg).
\end{align*}

Now, we can bound the regret by bounding the sum of expected instantaneous regrets throughout the time horizon:
\begin{align*}
Reg_{\bs{\mu},\alpha,\beta}^{\text{CUCB-}0}(T)
 &= \sum_{t=1}^{T} \mathbb{E}[\text{IR}_{\bs{\mu},\alpha,\beta}^{\text{CUCB-}0}(t)]
 \\
&\leq \sum_{t=1}^{\lceil t' \rceil} 
\mathbb{E}[\text{IR}_{\bs{\mu},\alpha,\beta}^{\text{CUCB-}0}(t)]
+\sum_{t=\lceil t' \rceil+1}^{T} \mathbb{E}[\text{IR}_{\bs{\mu},\alpha,\beta}^{\text{CUCB-}0}(t)] 
\\
& \leq \lceil t' \rceil \Delta_{\text{max}} +  \gamma(2m)^\omega\Bigg[ 2^\omega\Big(\frac{\pi}{A}\Big)^{\omega/2} +3^\omega \Bigg]\sum_{t=1}^{T}  t^{-\omega/2} .
\end{align*}

Using the fact that $\sum_{t=1}^{T} t^{-x} \leq T^{1-x}/(1-x)$ $\forall x \in (0,1)$, we have the following for $\omega \in (0,1]$:
\begin{align*}
Reg_{\bs{\mu},\alpha,\beta}^{\text{CUCB-}0}(T) &\leq \lceil t' \rceil \nabla_{\text{max}} +  \gamma(2m)^\omega\Bigg[ 2^\omega\Big(\frac{\pi}{A}\Big)^{\omega/2} +3^\omega\Bigg] \frac{T^{1-w/2}}{1-w/2} .
\end{align*}
Finally, since the above regret bound holds for any $\eta \in (0,1)$, we take the infimum over $\eta$.

\subsection{Proof of Theorem 5}

First, we let $F^{\text{Binom}}_{n,p}(x)$ be the cdf of Binomial$(n,p)$ and $F^{\text{Beta}}_{\alpha,\beta}(x)$ be the cdf of Beta$(\alpha,\beta)$. Then, we define a series of lemmas and facts that will be used in the proof.

\begin{lemma}(Lemma 9 in \citet{Agrawal2012})\label{BinomHoeffding}
	$F^{\text{Binom}}_{n,p}(np-n\delta)\leq e^{-2n\delta^2}$, $1-F^{\text{Binom}}_{n,p}(np+n\delta)\leq e^{-2n\delta^2}$ and $1-F^{\text{Binom}}_{n+1,p}(np+n\delta)\leq e^{4\delta-2n\delta^2}$  for all $n \in \mathbb{N}_+$, $p \in [0,1]$, $\delta \geq 0$.
\end{lemma}

\begin{fact}\label{BinomNPlus1toN}
$F^{\text{Binom}}_{n+1,p}(i)\geq F^{\text{Binom}}_{n,p}(i-1)$ for all $n \in \mathbb{N}_+$, $p \in [0,1]$, and integers $i$ such that $1 \leq i \leq n$.
\end{fact}
\begin{proof}
\begin{align*}
F^{\text{Binom}}_{n+1,p}(i) \geq (1-p)F^{\text{Binom}}_{n,p}(i)+pF^{\text{Binom}}_{n,p}(i-1)\geq F^{\text{Binom}}_{n,p}(i-1)
\end{align*}

\end{proof}

\begin{fact}(Fact 1 in \citet{Agrawal2012})\label{lemma:BetaBinom}
\begin{align*}
F^{\text{Beta}}_{\alpha,\beta}(y)=1-F^{\text{Binom}}_{\alpha+\beta-1,y}(\alpha-1)
\end{align*}
for all $\alpha,\beta \in \mathbb{N}_+$ and $y\in [0,1]$.
\end{fact}

Next, as in the CUCB-$\kappa$ case, we prove that $\nu_i^{t+1}$ is lower than some constant with high probability for $t$ sufficiently large. 
Recalling that $\Delta_{t+1}^0=|\hat{\mu}_i^{t+1}-\mu_i|$, we have the following by the law of total probability:
\begin{align}
 &\Pr(|\nu_i^{t+1}-\mu_i|\geq \delta) \notag \\
 &=\Pr(|\nu_i^{t+1}-\mu_i|\geq \delta,\Delta_{t+1}^0\geq \delta/2)+\Pr(|\nu_i^{t+1}-\mu_i|\geq \delta,\Delta_{t+1}^0<\delta/2) \notag \\&\leq \Pr(\Delta_{t+1}^0\geq \delta/2) +\Pr(|\nu_i^{t+1}-\mu_i|\geq \delta,\Delta_{t+1}^0<\delta/2) \label{eq:actualeq}. 
\end{align}

By Lemma \ref{lemma:sample}, we already have an upper bound for $\Pr(\Delta_{t+1}^0\geq \delta/2)$ for $t\geq t'$. Note that $|\nu_i^{t+1}-\mu_i|\geq \delta$ and $\Delta_{t+1}^0<\delta/2$ implies that $|\nu_i^{t+1}-\hat{\mu}_i^{t+1}|\geq \delta/2$ by the triangle inequality.
 Hence, we have 
\begin{align}
\Pr(\Delta_{t+1}^0< \delta/2, |\nu_i^{t+1}-\mu_i|\geq \delta)\leq \Pr(|\nu_i^{t+1}-\hat{\mu}_i^{t+1}|\geq \delta/2) \label{eq:remembernu}.
\end{align}
Now, we will find an upper bound for the RHS of this inequality. First, notice that we have the following
\begin{align}
&\Pr(|\nu_i^{t+1}-\hat{\mu}_i^{t+1}|\geq \delta/2)\notag\\&=\sum_{j=0}^{t} \Pr(|\nu_i^{t+1}-\hat{\mu}_i^{t+1}|\geq \delta/2| T_i^{t+1}=j)\Pr(T_i^{t+1}=j)\notag \\ 
&\leq\sum_{j=0}^{t} \Pr(\nu_i^{t+1}-\hat{\mu}_i^{t+1}\geq \delta/2| T_i^{t+1}=j)\Pr(T_i^{t+1}=j) \notag \\
&+ \sum_{j=0}^{t} \Pr(\nu_i^{t+1}-\hat{\mu}_i^{t+1}\leq -\delta/2 | T_i^{t+1}=j)\Pr(T_i^{t+1}=j) \label{eqn:thomphat}.
\end{align}

In the expression below, the second equality comes from the Fact \ref{lemma:BetaBinom} and the last inequality is obtained via the Lemma \ref{BinomHoeffding}:
 
\begin{align}
\Pr(\nu_i^{t+1}-\hat{\mu}_i^{t+1}\geq \delta/2| T_i^{t+1}=j)\notag &=\mathbb{E}[1-F^{\text{Beta}}_{1+s_i^{t+1},j-s_i^{t+1}+1}(\hat{\mu}_i^{t+1}+\delta/2)
|  T_i^{t+1}=j ] 
\notag \\&=\mathbb{E}[F^{\text{Binom}}_{j+1,\hat{\mu}_i^{t+1}+\delta/2}(s_i^{t+1})
| T_i^{t+1}=j] 
\notag \\&\leq \mathbb{E}[F^{\text{Binom}}_{j,\hat{\mu}_i^{t+1}+\delta/2}(s_i^{t+1}) |
T_i^{t+1}=j] \notag \\&= \mathbb{E}[F^{\text{Binom}}_{j,\hat{\mu}_i^{t+1}+\delta/2}(j\hat{\mu}_i^{t+1}) | T_i^{t+1}=j] \notag \\&\leq e^{-\delta^2 j/2}. \label{eq:rightnu}
\end{align}

Similarly, by Fact \ref{lemma:BetaBinom} and Lemma \ref{BinomHoeffding}, we also have 

\begin{align}
\Pr(\nu_i^{t+1}-\hat{\mu}_i^{t+1}\leq -\delta/2| T_i^{t+1}=j) &=\mathbb{E}[F^{\text{Beta}}_{1+s_i^{t+1},j-s_i^{t+1}+1}(\hat{\mu}_i^{t+1}-\delta/2) | T_i^{t+1}=j] \notag \\&=\mathbb{E}[1-F^{\text{Binom}}_{j+1,\hat{\mu}_i^{t+1}-\delta/2}(s_i^{t+1})| T_i^{t+1}=j] \notag \\
&\rev{=} \mathbb{E}[1-F^{\text{Binom}}_{j+1,\hat{\mu}_i^{t+1}-\delta/2}(j\hat{\mu}_i^{t+1}) |T_i^{t+1}=j] \notag \\
&\leq e^{2\delta-\delta^2j/2}. \label{eq:leftnu}
\end{align}

Using \eqref{eq:rightnu} and \eqref{eq:leftnu} in \eqref{eqn:thomphat}, we obtain
\begin{align*}
\Pr(|\nu_i^{t+1}-\hat{\mu}_i^{t+1}|\geq \delta/2)\leq (e^{2\delta}+1)\sum_{j=0}^{t} e^{-\delta^2 j /2} \Pr(T_i^{t+1}=j) . 
\end{align*}

For $0 \leq j < t^*=\eta p^* t$, we have $\Pr(T_i^{t+1}=j)\leq \Pr(T_i^{t+1}\leq t^{*}) \leq  1/t^2$ when $t\geq t'$ by Lemma \ref{lemma:plays2}. We also have $\sum_{j=\lceil t^*\rceil}^{t} \Pr(T_i^{t+1}=j)\leq 1$. Using these, we proceed as follows: For $t \geq t'$, we have
\begin{align}
	&\frac{\Pr(|\nu_i^{t+1}-\mu_i|\geq \delta,\Delta_{t+1}^0<\delta/2)}{e^{2\delta}+1}\notag\\&\leq\frac{\Pr(|\nu_i^{t+1}-\hat{\mu}_i^{t+1}|\geq \delta/2)}{e^{2\delta}+1} \notag\\&\leq \sum_{j=0}^{\lfloor t^*\rfloor} e^{-\delta^2 j /2}\Pr(T_i^{t+1}=j)+\sum_{j=\lceil t^*\rceil}^{t} e^{-\delta^2 j /2}\Pr(T_i^{t+1}=j) \notag\\& \leq \frac{1}{t^2}\sum_{j=0}^{\lfloor t^*\rfloor} e^{-\delta^2 j/2}+ e^{-\delta^2 \lceil t^*\rceil /2 }\sum_{j=\lceil t^*\rceil}^{t}\Pr(T_i^{t+1}=j)\notag\\&\leq \frac{1}{t^2(1-e^{-\delta^2/2})} + e^{-\delta^2\lceil t^*\rceil/2}\notag \\&\leq \frac{1}{t^2(1-e^{-\delta^2/2})} + e^{-\delta^2\eta p^* t/2}. \label{eq:nuActualEq}
\end{align}

Using the result of Lemma \ref{lemma:sample} and \eqref{eq:nuActualEq} in \eqref{eq:actualeq}, and recalling the inequality in \eqref{eq:remembernu}, we obtain
\begin{align*}
\Pr(|\nu_i^{t+1}-\mu_i|\geq \delta) \leq (3+e^{2\delta}) \Big[\frac{1}{t^2(1-e^{-\delta^2/2})}+e^{-\delta^2\eta p^* t/2}\Big]
\end{align*}
for $t\geq t'$, and for any $i\in \{1,\ldots,m\}$ and $\eta \in (0,1)$.

Finally by using the union bound, we obtain
\begin{align*}
\Pr \left(\bigcup_i\{|\nu_i^{t+1}-\mu_i|\geq \delta\} \right) &\leq (3+e^{2\delta})m \Big[\frac{1}{t^2(1-e^{-\delta^2/2})}+e^{-\delta^2\eta p^* t/2}\Big]
\end{align*}
for $t\geq t'$.

\subsection{Proof of Theorem 6}
	We follow the procedure we use in Theorem 3. Fix any $\eta \in (0,1)$. Note that $\delta >0$ since $\nabla_{\text{min}} >0$.  For $T \leq \lceil t' \rceil$, the regret bound is $T \nabla_{\text{max}}$. Recall that $\Delta_{t}^{\bs{\nu}}= \max_{i}|\nu_i^{t}-\mu_i|$. By Lemma \ref{lemma:gap}, we have
	\begin{align*}
	\{\Delta_{t}^{\bs{\nu}} < \delta\} &\Rightarrow \Pr(r_{\bs{\mu}}(S_t) \geq \alpha r_{\bs{\mu}}^{*})\geq \beta \\&\Rightarrow \mathbb{E}[r_{\bs{\mu}}(S_t)]\geq \alpha\beta r_{\bs{\mu}}^{*} \Rightarrow {\cal N}_t \\&\Rightarrow \mathbb{E}[\text{IR}_{\bs{\mu},\alpha,\beta}^{\text{CTS}}(t)] \leq 0 .
	\end{align*} 
	Hence, we have the following for $T >\lceil t' \rceil$:
	\begin{align*}
	&\mathbb{E} \left[\sum_{t=1}^{T}\text{IR}_{\bs{\mu},\alpha,\beta}^{\text{CTS}}(t) \right] 
	=\mathbb{E} \left[\sum_{t=1}^{\lceil t' \rceil}\text{IR}_{\bs{\mu},\alpha,\beta}^{\text{CTS}}(t) \right]
	+\mathbb{E}\left[ \sum_{t=\lceil t' \rceil+1}^{T}\text{IR}_{\bs{\mu},\alpha,\beta}^{\text{CTS}}(t) \right]  
	\\ &\leq \nabla_{\text{max}}\lceil t' \rceil+ \sum_{t=\lceil t^{'} \rceil+1}^{T}\mathbb{E}[\text{IR}_{\bs{\mu},\alpha,\beta}^{\text{CTS}}(t)|\Delta_{t}^{\bs{\nu}} < \delta]\Pr(\Delta_{t}^{\bs{\nu}} < \delta) \\
	&+ \sum_{t=\lceil t^{'} \rceil+1}^{T}\mathbb{E}[\text{IR}_{\bs{\mu},\alpha,\beta}^{\text{CTS}}(t)|\Delta_{t}^{\bs{\nu}} \geq \delta]\Pr(\Delta_{t}^{\bs{\nu}} \geq \delta) 
	\\ &= \nabla_{\text{max}}\lceil t' \rceil
	+ \sum_{t=\lceil t' \rceil+1}^{T}\mathbb{E}[\text{IR}_{\bs{\mu},\alpha,\beta}^{\text{CTS}}(t)|{\cal N}_t]\Pr(\Delta_{t}^{\bs{\nu}} < \delta) \\
	&+\sum_{t=\lceil t' \rceil+1}^{T}\mathbb{E}[\text{IR}_{\bs{\mu},\alpha,\beta}^{\text{CTS}}(t)|\Delta_{t}^{\bs{\nu}} \geq \delta]\Pr(\Delta_{t}^{\bs{\nu}} \geq \delta) 
	\\ &\leq \nabla_{\text{max}}\lceil t' \rceil+ \sum_{t=\lceil t' \rceil+1}^{T}\mathbb{E}[\text{IR}_{\bs{\mu},\alpha,\beta}^{\text{CTS}}(t)|\Delta_{t}^{\bs{\nu}} \geq \delta] \\
	&\times \left(\frac{(3+e^{2\delta})m}{(t-1)^2(1-e^{-\delta^2/2})} 
	+(3+e^{2\delta})m e^{-\delta^2\eta p^*(t-1)/2} \right) \\ 
	&\leq \nabla_{\text{max}}\lceil t' \rceil
	+ \nabla_{\text{max}} \sum_{t=1}^{\infty} \left(\frac{(3+e^{2\delta})m}{t^2(1-e^{-\delta^2/2})}+(3+e^{2\delta})m e^{-\delta^2\eta p^* t/2} \right)
	\\&\leq \nabla_{\text{max}} \left(\lceil t' \rceil+ \frac{(3+e^{2\delta})m\pi^2}{6(1-e^{-\delta^2/2})} 
	+\frac{(3+e^{2\delta})me^{-\delta^2\eta p^*/2}}{1-e^{-\delta^2\eta p^*/2}} \right)
	\\&\leq \nabla_{\text{max}} \left(\lceil t' \rceil+\frac{(3+e^{2\delta})m\pi^2}{6(1-e^{-\delta^2/2})} 
	+\frac{(3+e^{2\delta})m}{1-e^{-\delta^2\eta p^*/2}} \right).
	\end{align*}
	
	We know that $\frac{1}{1-e^{-x}}\leq 1+\frac{1}{x}$ for $x>0$. Using this, we obtain
	\begin{align}
	\mathbb{E} \left[\sum_{t=1}^{T}\text{IR}_{\bs{\mu},\alpha,\beta}^{\text{CTS}}(t) \right] 
 \leq \nabla_{\text{max}} \left(\lceil t' \rceil
 +\frac{(3+e^{2\delta})m\pi^2}{6} \left( 1+\frac{2}{\delta^2} \right) \notag
	+(3+e^{2\delta})m \left(1+\frac{2}{\delta^2\eta p^*} \right) \right) . \label{eq:etamin3}
	\end{align}
Finally, since the regret bound above holds for any $\eta \in (0,1)$, we take the infimum over $\eta$.

	\subsection{Proof of Theorem 7}
	For $T \leq \lceil t' \rceil$, the regret bound is $T \nabla_{\text{max}}$. To find the regret bound for $T > \lceil t' \rceil$,
first, we show that the instantaneous regret $\text{IR}_{\bs{\mu},\alpha,\beta}^{\text{CTS}}(t)$ can be bounded by using $\Delta_{t}^{\bs{\nu}}$ and the bounded smoothness function $f$.
\begin{lemma} \label{lemma:al2}
For any integer $t\geq 1$, expectation vector $\bs{\mu}$, the instantaneous regret of CTS is bounded as follows:
\begin{align*}
\mathbb{E}[\text{IR}_{\bs{\mu},\alpha,\beta}^{\text{CTS}}(t)]\leq 2\mathbb{E}[f(\Delta_{t}^{\bs{\nu}})].
\end{align*}

\end{lemma}
\begin{proof}
Taking expectation of both sides of the equation \eqref{eqn:lemma1} in Lemma \ref{lemma:gap}, we obtain the desired result.
\end{proof}
The next step is to bound $\mathbb{E}[\Delta_{t+1}^{\bs{\nu}}]$ for $t>\lceil t' \rceil$.

\begin{lemma} \label{UCBplusSample2}
For any $\eta \in (0,1)$, we have
\begin{align*}
\mathbb{E}[\Delta_{t+1}^{\bs{\nu}}] \leq (3+e^{2})mt^{-1/2}\Big(\sqrt{\frac{2\pi}{\eta p^*}}+ 3 \Big)
\end{align*}
for all integers $t \geq t' := 4c^2/e^2$, where $c := 1/(p^*(1-\eta))^2$.
\end{lemma}
\begin{proof}
We have
\begin{align*}
\mathbb{E}[\Delta_{t+1}^{\bs{\nu}}] = \int_{0}^{1} \Pr(\Delta_{t+1}^{\bs{\nu}} \geq \delta)\text{d}\delta.
\end{align*}
Hence, by Theorem \ref{thm:deltathompson}, since $\Delta_{t+1}^{\bs{\nu}} \in [0,1]$, we have for any $\theta>0$
\begin{align*}
\mathbb{E}[\Delta_{t+1}^{\bs{\nu}}]  
&\leq \int_{0}^{t^{-\theta}} \Pr(\Delta_{t+1}^{\bs{\nu}} \geq \delta)\text{d}\delta+\int_{t^{-\theta}}^{1} \Big(\frac{(3+e^{2\delta})m}{t^2(1-e^{-\delta^2/2})} +(3+e^{2\delta})m e^{-\delta^2\eta p^* t/2}\Big)\text{d}\delta \\
&\leq t^{-\theta} +\int_{t^{-\theta}}^{1} \Big(\frac{(3+e^{2})m}{t^2(1-e^{-\delta^2/2})} +(3+e^{2})m e^{-\delta^2\eta p^* t/2}\Big)\text{d}\delta
\end{align*}
for $t>\lceil t' \rceil$.

Note that $e^{-\delta^2\eta p^* t/2}$ is just a scaled version of the probability distribution function of normal distribution. Let $B=\sqrt{1/(\eta p^*t)}$, then we have

\begin{align*}
\int e^{-\delta^2\eta p^* t/2}\text{d}\delta = \int e^{-\delta^2/(2B^2)}\text{d}\delta = B\sqrt{2\pi} \int \frac{1}{ B\sqrt{2\pi}} e^{\frac{-\delta^2}{2B^2}}\text{d}\delta.
\end{align*} 
Hence, setting $B=\sqrt{1/(2At)}$ ($A=\eta p^*/2$), via the fact that $1/(1-e^{-x})\leq 1 + 1/x$ for all $x>0$, we obtain
\begin{align*} 
\mathbb{E}[\Delta_{t+1}^{\bs{\nu}}] &\leq t^{-\theta}+ (3+e^{2})m\sqrt{\frac{\pi}{At}}+ \int_{t^{-\theta}}^{1} \Big(\frac{(3+e^{2})m}{t^2(1-e^{-\delta^2/2})} \Big)\text{d}\delta \notag \\&\leq t^{-\theta}+ (3+e^{2})m\sqrt{\frac{\pi}{At}}+ \frac{(3+e^{2})m}{t^2}\int_{t^{-\theta}}^{1} \Big(1+\frac{2}{\delta^2}\Big)\text{d}\delta \notag \\&=t^{-\theta}+(3+e^{2})m\sqrt{\frac{\pi}{At}}+(3+e^{2})mt^{-2}\bigg(\delta-\frac{2}{\delta}\bigg)\bigg|_{\delta=t^{-\theta}}^{\delta=1} \notag \\&=t^{-\theta}+(3+e^{2})m\sqrt{\frac{\pi}{At}}\\&+(3+e^{2})mt^{-2}\bigg(1-2-t^{-\theta}+2t^{\theta}\bigg)\notag 
\\&\leq t^{-\theta}+ (3+e^{2})m\sqrt{\frac{\pi}{At}}+ 2(3+e^{2})mt^{\theta-2}.  \notag
\end{align*}
The order of $t$ in this bound is minimized when $\theta=1$; thus, we have 

	\begin{align} 
\mathbb{E}[\Delta_{t+1}^{\bs{\nu}}] 
&\leq (3+e^{2})m\sqrt{\frac{\pi}{At}} + (2(3+e^{2})m + 1) t^{-1}
\\&\leq (3+e^{2})mt^{-1/2}\sqrt{\frac{\pi}{A}}+3(3+e^{2})mt^{-1/2} \notag 
 \\  &\leq (3+e^{2})mt^{-1/2}\Big(\sqrt{\frac{\pi}{A}}+3\Big) . \label{eq:gapindthomp}
\end{align} 

\end{proof}
Next, using Lemma \ref{UCBplusSample2}, we show that the gap-independent regret for CTS is at most $O(T^{1-w/2})$ for a large set of problems where the bounded-smoothness function $f(x)$ is $\gamma x^\omega$, for $\omega \in (0,1]$. 

Note that $f$ is concave. Hence, via Lemma \ref{lemma:al2}, \eqref{eq:gapindthomp} and Jensen's inequality, we have
\begin{align*}
\mathbb{E}[\text{IR}_{\bs{\mu},\alpha,\beta}^{\text{CTS}}(t+1)] &\leq 2\mathbb{E}[f(\Delta_{t+1}^{\bs{\nu}})] \\& \leq
2f(\mathbb{E}[\Delta_{t+1}^{\bs{\nu}}])
\\&\leq 2f\bigg((3+e^{2})m\sqrt{\frac{\pi}{At}} +3(3+e^{2})m\sqrt{\frac{1}{t}} \bigg) 
\\&\leq f\Big(2(3+e^{2})m\sqrt{\frac{\pi}{At}}\Big) 
+ f\Big(6(3+e^{2})m\sqrt{\frac{1}{t}}\Big)  
\\&=\gamma\bigg(\big(2(3+e^{2})m\big)^\omega \Big(\frac{\pi}{At}\Big)^{\omega/2}
+\big(6(3+e^{2})m\big)^{\omega} t^{-\omega/2} \bigg).
\end{align*}

Now, we can bound the regret by bounding the sum of expected instantaneous regrets throughout the time horizon:
\begin{align*}
Reg_{\bs{\mu},\alpha,\beta}^{\text{CTS}}(T)
 &= \sum_{t=1}^{T} \mathbb{E}[\text{IR}_{\bs{\mu},\alpha,\beta}^{\text{CTS}}(t)]
\\&\leq \sum_{t=1}^{\lceil t' \rceil} 
\mathbb{E}[\text{IR}_{\bs{\mu},\alpha,\beta}^{\text{CTS}}(t)]
+\sum_{t=\lceil t' \rceil + 1}^{T} \mathbb{E}[\text{IR}_{\bs{\mu},\alpha,\beta}^{\text{CTS}}(t)] 
\\& \leq \lceil t' \rceil \nabla_{\max} 
+  \gamma\Big(2m(3+e^2)\Big)^\omega\Bigg[ \Big(\frac{\pi}{A}\Big)^{\omega/2} 
+3^\omega \Bigg]\sum_{t=1}^{T}  t^{-\omega/2}.
\end{align*}

Using the fact that $\sum_{t=1}^{T} t^{-x} \leq T^{1-x}/(1-x)$ $\forall x \in (0,1)$, we have the following for $\omega \in (0,1]$:
\begin{align*}
Reg_{\bs{\mu},\alpha,\beta}^{\text{CTS}}(T) &\leq \lceil t' \rceil \nabla_{\max} +   \gamma\Big(2m(3+e^2)\Big)^\omega\Bigg[ \Big(\frac{\pi}{A}\Big)^{\omega/2} +3^\omega \Bigg]\frac{T^{1-w/2}}{1-w/2}.
\end{align*}

\bibliographystyle{spbasic}
\bibliography{references}

\begin{thebibliography}{38}
\providecommand{\natexlab}[1]{#1}
\providecommand{\url}[1]{{#1}}
\providecommand{\urlprefix}{URL }
\expandafter\ifx\csname urlstyle\endcsname\relax
  \providecommand{\doi}[1]{DOI~\discretionary{}{}{}#1}\else
  \providecommand{\doi}{DOI~\discretionary{}{}{}\begingroup
  \urlstyle{rm}\Url}\fi
\providecommand{\eprint}[2][]{\url{#2}}

\bibitem[{Abbasi-Yadkori et~al(2011)Abbasi-Yadkori, P\'{a}l, and
  Szepesv\'{a}ri}]{Abbasi2011}
Abbasi-Yadkori Y, P\'{a}l D, Szepesv\'{a}ri C (2011) Improved algorithms for
  linear stochastic bandits. In: \textit{Proc. Advances in Neural Information
  Processing Systems (NIPS)}, pp 2312--2320

\bibitem[{Agrawal(1995)}]{agrawal1995sample}
Agrawal R (1995) Sample mean based index policies by ${O}(\log n)$ regret for
  the multi-armed bandit problem. \textit{Advances in Applied Probability}
  27(4):1054--1078

\bibitem[{Agrawal and Goyal(2012)}]{Agrawal2012}
Agrawal S, Goyal N (2012) Analysis of {Thompson} sampling for the multi-armed
  bandit problem. In: \textit{Proc. Conference on Learning Theory (COLT)}, pp
  39.1--39.26

\bibitem[{Agrawal et~al(2017)Agrawal, Avadhanula, Goyal, and
  Zeevi}]{agrawal2017thompson}
Agrawal S, Avadhanula V, Goyal V, Zeevi A (2017) Thompson sampling for the
  {MNL}-bandit. \textit{arXiv preprint arXiv:170600977}

\bibitem[{Akbarzadeh and Tekin(2016)}]{Nima}
Akbarzadeh N, Tekin C (2016) Gambler's ruin bandit problem. In: \textit{Proc.
  54th Annual Allerton Conference on Communication, Control, and Computing}, pp
  1236--1243

\bibitem[{Anantharam et~al(1987)Anantharam, Varaiya, and
  Walrand}]{anantharam1987asymptotically}
Anantharam V, Varaiya P, Walrand J (1987) Asymptotically efficient allocation
  rules for the multiarmed bandit problem with multiple plays-{Part} {I}: {IID}
  rewards. \textit{IEEE Trans Autom Control} 32(11):968--976

\bibitem[{Atan et~al(2015)Atan, Tekin, and Schaar}]{atan2015global}
Atan O, Tekin C, Schaar Mvd (2015) {Global multi-armed bandits with H{\"o}lder
  continuity}. In: \textit{Proc. AISTATS}, pp 28--36

\bibitem[{Auer(2003)}]{Auer2003}
Auer P (2003) Using confidence bounds for exploitation-exploration trade-offs.
  \textit{J Mach Learn Res (JMLR)} 3:397--422

\bibitem[{Auer et~al(2002)Auer, Cesa-Bianchi, and Fischer}]{Auer2002}
Auer P, Cesa-Bianchi N, Fischer P (2002) Finite-time analysis of the multiarmed
  bandit problem. \textit{Machine Learning} 47(2-3):235--256

\bibitem[{Bubeck and Cesa-Bianchi(2012)}]{Bubeck2012}
Bubeck S, Cesa-Bianchi N (2012) Regret analysis of stochastic and nonstochastic
  multi-armed bandit problems. \textit{Foundations and Trends in Machine
  Learning} 5(1):1--122

\bibitem[{Chen et~al(2013)Chen, Wang, and Yuan}]{chen2013combinatorial}
Chen W, Wang Y, Yuan Y (2013) Combinatorial multi-armed bandit: General
  framework and applications. In: \textit{Proc. Int. Conf. Machine Learning},
  pp 151--159

\bibitem[{Chen et~al(2016{\natexlab{a}})Chen, Hu, Li, Li, Liu, and
  Lu}]{Chen2016b}
Chen W, Hu W, Li F, Li J, Liu Y, Lu P (2016{\natexlab{a}}) Combinatorial
  multi-armed bandit with general reward functions. In: \textit{Proc. Advances
  in Neural Information Processing Systems (NIPS)}, pp 1651--1659

\bibitem[{Chen et~al(2016{\natexlab{b}})Chen, Wang, Yuan, and Wang}]{Chen2016a}
Chen W, Wang Y, Yuan Y, Wang Q (2016{\natexlab{b}}) Combinatorial multi-armed
  bandit and its extension to probabilistically triggered arms. \textit{J Mach
  Learn Res (JMLR)} 17(1):1746--1778

\bibitem[{Dani et~al(2008)Dani, Hayes, and Kakade}]{Dani2008}
Dani V, Hayes TP, Kakade SM (2008) Stochastic linear optimization under bandit
  feedback. In: \textit{Proc. Conference on Learning Theory (COLT)}, pp
  355--366

\bibitem[{Durand and Gagn{\'e}(2014)}]{durand2014thompson}
Durand A, Gagn{\'e} C (2014) Thompson sampling for combinatorial bandits and
  its application to online feature selection. In: \textit{Proc. AAAI-14
  Workshop on Sequential Decision-making with Big Data}

\bibitem[{Gai et~al(2012)Gai, Krishnamachari, and Jain}]{Gai2012}
Gai Y, Krishnamachari B, Jain R (2012) Combinatorial network optimization with
  unknown variables: Multi-armed bandits with linear rewards and individual
  observations. \textit{IEEE/ACM Trans Netw} 20(5):1466--1478

\bibitem[{Gopalan et~al(2014)Gopalan, Mannor, and Mansour}]{Gopalan2014}
Gopalan A, Mannor S, Mansour Y (2014) Thompson sampling for complex online
  problems. In: \textit{Proc. 31st Int. Conf. Machine Learning (ICML)}, vol~32,
  pp 100--108

\bibitem[{Graepel et~al(2010)Graepel, Quiñonero~Candela, Borchert, and
  Herbrich}]{Graepel2010}
Graepel T, Quiñonero~Candela J, Borchert T, Herbrich R (2010) Web-scale
  {Bayesian} click-through rate prediction for sponsored search advertising in
  {Microsoft’s Bing Search Engine}. In: \textit{Proc. 27th Int. Conf. Machine
  Learning (ICML)}, pp 13--20

\bibitem[{Granmo(2010)}]{Granmo2010}
Granmo O (2010) Solving two-armed {Bernoulli} bandit problems using a
  {Bayesian} learning automaton. \textit{Int J Intelligent Computing and
  Cybernetics} 3(2):207--234

\bibitem[{Kalman(2001)}]{Kalman2001}
Kalman D (2001) A generalized logarithm for exponential-linear equations.
  \textit{The College Mathematics Journal} 32(1)

\bibitem[{Kaufmann et~al(2012)Kaufmann, Korda, and Munos}]{Kaufmann2012}
Kaufmann E, Korda N, Munos R (2012) Thompson sampling: An optimal finite time
  analysis. In: \textit{Proc. Int. Conf. Algorithmic Learning Theory (ALT)}, pp
  199--213

\bibitem[{Kempe et~al(2003)Kempe, Kleinberg, and Tardos}]{Kempe2003}
Kempe D, Kleinberg J, Tardos E (2003) Maximizing the spread of influence
  through a social network. In: \textit{Proc. 9th ACM SIGKDD Int. Conf.
  Knowledge Discovery and Data Mining}, pp 137--146

\bibitem[{Kveton et~al(2014)Kveton, Wen, Ashkan, Eydgahi, and
  Eriksson}]{Kveton2014}
Kveton B, Wen Z, Ashkan A, Eydgahi H, Eriksson B (2014) Matroid bandits: Fast
  combinatorial optimization with learning. \textit{arXiv preprint
  arXiv:14035045}

\bibitem[{Kveton et~al(2015{\natexlab{a}})Kveton, Szepesvari, Wen, and
  Ashkan}]{Kveton2015a}
Kveton B, Szepesvari C, Wen Z, Ashkan A (2015{\natexlab{a}}) Cascading bandits:
  Learning to rank in the cascade model. In: \textit{Proc. 32nd Int. Conf.
  Machine Learning (ICML)}, pp 767--776

\bibitem[{Kveton et~al(2015{\natexlab{b}})Kveton, Wen, Ashkan, and
  Szepesvari}]{Kveton2015c}
Kveton B, Wen Z, Ashkan A, Szepesvari C (2015{\natexlab{b}}) Combinatorial
  cascading bandits. In: \textit{Proc. Advances in Neural Information
  Processing Systems (NIPS)}, pp 1450--1458

\bibitem[{Kveton et~al(2015{\natexlab{c}})Kveton, Wen, Ashkan, and
  Szepesvári}]{Kveton2015b}
Kveton B, Wen Z, Ashkan A, Szepesvári C (2015{\natexlab{c}}) Tight regret
  bounds for stochastic combinatorial semi-bandits. In: \textit{Proc. AISTATS},
  vol~38, pp 535--543

\bibitem[{Lai and Robbins(1985)}]{Lai1985}
Lai T, Robbins H (1985) Asymptotically efficient adaptive allocation rules.
  \textit{Advances in Applied Mathematics} 6(1):4--22

\bibitem[{Lei et~al(2015)Lei, Maniu, Mo, Cheng, and Senellart}]{Lei2015}
Lei S, Maniu S, Mo L, Cheng R, Senellart P (2015) Online influence
  maximization. In: \textit{Proc. 21th ACM SIGKDD Int. Conf. Knowledge
  Discovery and Data Mining}, pp 645--654

\bibitem[{Mersereau et~al(2009)Mersereau, Rusmevichientong, and
  Tsitsiklis}]{mersereau2009structured}
Mersereau AJ, Rusmevichientong P, Tsitsiklis JN (2009) A structured multiarmed
  bandit problem and the greedy policy. \textit{IEEE Trans Autom Control}
  54(12):2787--2802

\bibitem[{Nemhauser et~al(1978)Nemhauser, Wolsey, and Fisher}]{Nemhauser1978}
Nemhauser GL, Wolsey LA, Fisher ML (1978) An analysis of approximations for
  maximizing submodular set functions--{I}. \textit{Mathematical Programming}
  14(1):265--294

\bibitem[{Robbins(1952)}]{Robbins1952}
Robbins H (1952) Some aspects of the sequential design of experiments.
  \textit{Bulletin of the American Mathematical Society} 58(5):527--535

\bibitem[{Sarıtaç et~al(2016)Sarıtaç, Karakurt, and Tekin}]{biz}
Sarıtaç A, Karakurt A, Tekin C (2016) Online contextual influence
  maximization in social networks. In: \textit{Proc. 54th Annual Allerton
  Conference on Communication, Control, and Computing}, pp 1204--1211

\bibitem[{Scott(2010)}]{Scott2010}
Scott SL (2010) A modern {Bayesian} look at the multi-armed bandit.
  \textit{Applied Stochastic Models in Business and Industry} 26(6):639--658

\bibitem[{Thompson(1933)}]{thompson1933likelihood}
Thompson WR (1933) On the likelihood that one unknown probability exceeds
  another in view of the evidence of two samples. \textit{Biometrika}
  25(3/4):285--294

\bibitem[{Vaswani et~al(2015)Vaswani, Lakshmanan, Schmidt et~al}]{Vaswani2015}
Vaswani S, Lakshmanan L, Schmidt M, et~al (2015) Influence maximization with
  bandits. \textit{arXiv preprint arXiv:150300024}

\bibitem[{Vazirani(2001)}]{Vazirani2001}
Vazirani VV (2001) \textit{Approximation algorithms}. Springer

\bibitem[{Wang and Chen(2017)}]{Wang2017}
Wang Q, Chen W (2017) Tighter regret bounds for influence maximization and
  other combinatorial semi-bandits with probabilistically triggered arms.
  \textit{arXiv preprint arXiv:170301610}

\bibitem[{Wen et~al(2017)Wen, Kveton, and Valko}]{Wen2017}
Wen Z, Kveton B, Valko M (2017) Online influence maximization under independent
  cascade model with semi-bandit feedback. \textit{arXiv preprint
  arXiv:160506593}

\end{thebibliography}

\end{document}